%% file: main.tex
\documentclass[twoside]{article}

\usepackage{fullpage}
\usepackage{authblk}
\usepackage{parskip}

\usepackage[backend=biber,natbib=true,style=authoryear,doi=false,isbn=false,url=false,eprint=false,giveninits=true,maxbibnames=200,maxcitenames=2,mincitenames=1,uniquename=false,uniquelist=false,dashed=false]{biblatex}

\DeclareSourcemap{
  \maps[datatype=bibtex]{
    \map{
      \step[fieldsource=url, match=\regexp{http://(dx.doi.org/|dl.acm.org/)}, final]
      \step[fieldset=url, null]
      \step[fieldset=urldate, null]
    }
  }
}
\usepackage{amsmath}
\usepackage{amssymb}
\usepackage{amsthm}

\usepackage{subcaption}

\usepackage{algorithm}
\usepackage[noend]{algpseudocode}

\usepackage{enumitem}
\usepackage{hyperref}
\usepackage{cleveref}

\usepackage{shortcuts}
\usepackage{shortcuts_neurips2022}

\usepackage{tabularx,booktabs}
\usepackage{csvsimple}

\definecolor{linkcolor}{RGB}{83,83,182}

\hypersetup{
  colorlinks=true,
  citecolor=linkcolor,
  linkcolor=linkcolor,
  urlcolor=linkcolor
}

\begin{document}

\title{High-Dimensional Private Empirical Risk Minimization\\ by Greedy Coordinate Descent}
\date{}

\author[1]{Paul Mangold}
\author[1]{Aurélien Bellet}
\author[2,3]{Joseph Salmon}
\author[4]{Marc Tommasi}

\affil[1]{Univ. Lille, Inria,  CNRS, Centrale Lille, UMR 9189 - CRIStAL, F-59000 Lille, France}
\affil[2]{IMAG, Univ Montpellier, CNRS, Montpellier, France}
\affil[3]{Institut Universitaire de France (IUF)}
\affil[4]{Univ. Lille, CNRS, Inria, Centrale Lille,  UMR 9189 - CRIStAL, F-59000 Lille, France}

\maketitle

\begin{abstract}
  In this paper, we study differentially private empirical risk
  minimization (DP-ERM). It has been shown that the worst-case
  utility of DP-ERM reduces polynomially as the dimension increases. This is a
  major obstacle to privately learning large machine learning
  models. In high dimension, it is common for some model's parameters
  to carry more information than others. To exploit this, we propose a
  differentially private greedy coordinate descent (DP-GCD)
  algorithm. At each iteration, DP-GCD privately performs a
  coordinate-wise gradient step along the gradients' (approximately)
  greatest entry. We show theoretically that DP-GCD can achieve a
  logarithmic dependence on the dimension for a wide range of problems by
  naturally exploiting their structural properties (such as
  quasi-sparse solutions). We
  illustrate this behavior numerically, both on synthetic and real datasets.
\end{abstract}

\input{intro}

\input{related-works}

\input{prelim}

\input{greedy-cd}

\input{expe}

\input{conclu}

\section*{Acknowledgments}

The authors would like to thank the anonymous reviewers who provided
useful feedback on previous versions of this work, which helped to improve
the paper.

\looseness=-1 This work was supported by the Inria Exploratory Action FLAMED
and by the French National Research Agency (ANR) through grant ANR-20-CE23-0015
(Project PRIDE), ANR-20-CHIA-0001-01 (Chaire IA CaMeLOt) and ANR 22-PECY-0002 IPOP
(Interdisciplinary Project on Privacy) project of the Cybersecurity PEPR.

\printbibliography

\appendix
\onecolumn

\input{sup-proof-privacy}

\input{sup-proof-utility}

\input{sup-fast-initial}

\input{sup-proximal-gcd}

\input{sup-expe-details}

\input{sup-expe-time}

\end{document}

%% file: intro.tex
\section{Introduction}

Machine Learning (ML) crucially relies on data, which can be sensitive or
confidential. Unfortunately, trained models are prone to leaking
information about specific training points
\citep{shokri2017Membership}. A standard approach for training models
while provably controlling the amount of leakage is to solve an
empirical risk minimization (ERM) problem under differential privacy
(DP) constraints \citep{chaudhuri2011Differentially}. In this work, we
consider the generic problem formulation:
\begin{align}
  \label{eq:dp-erm}
  w^* \in
  \argmin_{w \in \mathbb{R}^p}
  \Big\{ f(w) = \frac{1}{n} \sum_{i=1}^n \ell(w; d_i) \Big\} \enspace,
\end{align}
where $D = (d_1, \dots, d_n)$
is a dataset of $n$ samples drawn from a universe $\cX$, and
$\ell(\cdot, d): \RR^p \rightarrow \RR$ is a loss function which is
convex and smooth for all $d\in D$.

The DP constraint in DP-ERM induces a trade-off between the precision
of the solution (utility) and privacy.
\citet{bassily2014Differentially} proved lower bounds on utility under
a fixed DP budget. These lower bounds scale polynomially with the
dimension $p$. Since machine learning models are often
high-dimensional (\eg $n \approx p$ or even $n \ll p$), this is a
massive drawback for the use of DP-ERM.

To go beyond this negative result, one has to leverage the fact
that high-dimensional problems often exhibit
some \emph{structure}. In
particular, some parameters are typically more
significant than others: it is notably (but not only) the case when
models are sparse, which is often desired in high dimension
\citep{tibshirani1996Regression}. Private learning algorithms could
thus be designed to exploit this by focusing on the most significant
parameters of the problem.
Several works have tried to exploit such high-dimensional problems' structure
to reduce the dependence on the dimension, \eg from polynomial to
logarithmic.
\citet{talwar2015Nearly,bassily2021NonEuclidean,asi2021Private}
proposed a DP Frank-Wolfe algorithm (DP-FW) that exploits the solution's
sparsity. However, their algorithm only works on $\ell_1$-constrained
DP-ERM, restricting its range of application.
For sparse linear regression,
\citet{kifer2012Private} proposed to first identify some support and
then solve the DP-ERM problem on the restricted support.
Unfortunately, their approach requires implicit knowledge of the
solution's sparsity.
Finally,
\citet{kairouz2021Nearly,zhou2021Bypassing} used  public data to estimate
lower-dimensional subspaces, where the gradient can
be computed at a reduced privacy cost. A key limitation is that such public
data set, from the same domain as the private data, is typically not available
in many
learning scenarios involving sensitive data.

In this work, we propose a private algorithm that does not have these
pitfalls: the differentially private greedy coordinate descent
algorithm (DP-GCD). At each iteration, DP-GCD privately determines the
gradient's greatest coordinate, and performs a gradient step in its
direction. It focuses on the most useful parameters, avoiding wasting
privacy budget on updating non-significant ones.  Our algorithm
works on any smooth, unconstrained DP-ERM problem. We also propose a proximal
version to tackle non-smooth regularizers. Crucially, DP-GCD is adaptive to
the sparsity of the solution, and is able to ignore small (but non-zero)
parameters, improving utility even on non-sparse problems.

Formally, we show that DP-GCD reduces the dependence on the dimension
from $\sqrt{p}$ or $p$ to $\log(p)$ for a wide range of unconstrained
problems. This is the first algorithm to obtain such gains without
relying on $\ell_1$ or $\ell_0$ constraints. In fact, DP-GCD's utility
naturally depends on $\ell_1$-norm quantities (\ie distance from
initialization to optimal or strong-convexity parameter) and spans two
different regimes. When these $\ell_1$-norm quantities are $O(1)$ as
assumed in DP-FW, DP-GCD attains $O(\log(p)/n^{2/3}\epsilon^{2/3})$
and $O(\log(p)/n^2\epsilon^2)$ utility on convex and strongly-convex
problems respectively, outperforming existing DP-FW algorithms without
solving a constrained problem. In the second regime, when the
$\ell_2$-norm counterpart of the above quantities are $O(1)$ as
assumed for DP-SGD and its variants, we show that DP-GCD adapts to the
problem's underlying structure. Specifically, it is able to
\emph{interpolate between logarithmic and polynomial dependence on the
  dimension}.  In addition to these general utility results, we prove
that for strongly convex problems with quasi-sparse solutions
(including but not limited to sparse problems), DP-GCD converges to a
good approximate solution in few iterations. This improves utility in
the $\ell_2$-norm setting, replacing the polynomial dependence on the
ambient space's dimension by the quasi-sparsity level of the solution.
We evaluate both our
algorithms numerically on real and synthetic datasets, validating our
theoretical observations.

Our contributions can be summarized as follows:
\begin{enumerate}
  \item We propose differentially private greedy coordinate descent
        (DP-GCD), a method that performs
        updates along the (approximately) greatest entry of the gradient. We
        formally establish
        its privacy guarantees, and derive high
        probability utility upper bounds.
  \item We prove that DP-GCD exploits structural
        properties of the problem (\eg quasi-sparse solutions) to
        improve utility. Importantly, DP-GCD
        does not require prior knowledge of this structure to exploit
        it.
  \item We empirically validate our theoretical results on a variety of
        synthetic and real datasets, showing that DP-GCD outperforms
        existing private algorithms on high-dimensional problems with
        quasi-sparse solutions.
\end{enumerate}

The rest of the paper is organized as follows. First, we discuss
related work in more details in \Cref{sec:related-works}, and present the
relevant mathematical background in
\Cref{sec:preliminaries}. \Cref{sec:priv-greedy-coord} then introduces
DP-GCD, and formally analyzes its privacy and utility. We validate our
theoretical results numerically in \Cref{sec:experiments-1}. Finally,
we conclude and discuss the limitations of our results in
\Cref{sec:conclusion-and-discussion}.

%% file: related-works.tex
\section{Related Work}
\label{sec:related-works}

\paragraph{DP-ML in Euclidean geometry}

Most of the work on differentially private empirical risk minimization
(DP-ERM) and differentially private stochastic convex optimization
(DP-SCO)\footnote{See \citep{dwork2015Preserving,bassily2016Algorithmic,jung2021New}
for techniques to convert DP-ERM results to DP-SCO.}
consider problem quantities (\eg bounds on the domain and regularity
assumptions) expressed
in $\ell_2$ norm. In this Euclidean setting, \citet{bassily2014Differentially}
analyzed the theoretical properties of DP-SGD for DP-ERM, and
derived matching utility lower bounds. Faster algorithms
based on SVRG \citep{johnson2013Accelerating,xiao2014Proximal} were designed by
\citet{wang2017Differentially}.
\citet{wu2017Bolton} studied a variant of DP-SGD with
output perturbation, that is efficient when only few passes on the
data are possible. For DP-SCO,
\citet{bassily2019Private} used algorithmic stability
arguments \citep[following work
    from][]{hardt2016Train,bassily2020Stability} to show that in some regimes, the
population risk is the same as in non-private
SCO. \citet{feldman2020Private,wang2022Differentially} then developed
efficient (linear-time) algorithm to solve this problem. In all of the
above work, \emph{the utility upper bounds scale polynomially in $p$}, which
is not suitable in high dimension. In contrast, our approach provably achieves
logarithmic dependence on the dimension for some problems.

\paragraph{DP-ML in high dimension}
Several approaches have been explored to reduce the dependence on the
dimension. One option is to consider $\ell_1$-constrained problems.
For DP-ERM, \citet{talwar2015Nearly,talwar2016Private} used a
differentially private Frank-Wolfe algorithm (DP-FW)
\citep{frank1956Algorithm,jaggi2013Revisiting} to achieve utility that
scales logarithmically with the
dimension. \citet{asi2021Private,bassily2021NonEuclidean} proposed
stochastic DP-FW algorithms, extending the above results to
DP-SCO. For more general domains (\eg polytopes),
\citet{kasiviswanathan2016Efficient} randomly project the data on a
smaller-dimensional space, and lift the result back onto the original
space. The dependence in the dimension is encoded by the Gaussian
width of the domain, again leading to $O(\log p)$ error for the
$\ell_1$ ball or the simplex. \citet{wang2017Differentially} derived a
faster mirror descent algorithm for DP-ERM whose utility also depends
on the Gaussian width of the domain. Our approach matches the
$O(\log p)$ dependence of the above methods when key quantities are
bounded in $\ell_1$ norm, but can also achieve such gains for more
general problems, \eg when the problem has a quasi-sparse solution.
\citet{kifer2012Private} previously leveraged the solution sparsity
for the specific problem of sparse linear regression: they first
identify some support, and then solve DP-ERM on this restricted
support. Similarly, \citet{wang2019Differentially,hu2022High} proposed
hard thresholding-based algorithms for DP-ERM and DP-SCO under
sparsity ($\ell_0$ norm) constraints. Both approaches achieve an error of $O
(\log
p)$
but rely either on prior knowledge on the solution's sparsity, or on
the tuning of an additional hyperparameter. In contrast, our approach
automatically adapts to the sparsity and works also when solutions are
only quasi-sparse.  Finally,
\citet{kairouz2021Nearly,zhou2021Bypassing} estimate lower-dimensional
gradient subspaces using public data. This reduces noise addition, but
in practice, public data is only rarely available.

\paragraph{Coordinate descent}
CD algorithms have a long
history in optimization \citep[see][for detailed reviews on CD]
{wright2015Coordinate,shi2017Primer}. Most approaches have focused on
randomized or cyclic choices of coordinates
\citep{tseng2001Convergence,nesterov2010Efficiency}, with proximal and
parallel variants
\citep{richtarik2014Iteration,fercoq2014Accelerated,hanzely2018SEGA},
sometimes applied to the dual problem
\citep{shalev-shwartz2013Stochastic}.
In this work, our focus is on greedy coordinate descent methods, which update
the coordinate with greatest gradient entry
\citep{luo1992Convergence,tseng2009Coordinate,dhillon2011Nearest}.
\citet{nutini2015Coordinate} showed
improved convergence
rates for smooth, strongly-convex functions, by measuring strong
convexity in the $\ell_1$-norm. Our work builds upon these results to design
and analyze the first \emph{private} greedy CD approach.
Although techniques such as fast nearest-neighbor schemes have been proposed
to compute the (approximate) greedy update more efficiently
\citep{dhillon2011Nearest,nutini2015Coordinate,karimireddy2019Efficient},
greedy CD methods are often slower (in wall-clock
time) than their randomized or cyclic counterparts \citep{Massias_Gramfort_Salmon17b}.
However, in the private setting we consider, the main focus is not on
computing time but on achieving the best privacy-utility trade-off. This gives
a distinct advantage to greedy CD, as it provides a way to perform
the (approximately) most useful coordinate update under a given privacy
budget instead of wasting budget on updating random (potentially useless)
coordinates.
The analysis of proximal extensions of greedy CD for composite problems with
non-smooth parts is known to be challenging even in the non-private setting.
\citet{karimireddy2019Efficient}
proved convergence rates only for $\ell_1$- and box-regularized problems,
using a modified greedy CD algorithm. In this work, we propose and empirically
evaluate a proximal
extension of our DP-GCD algorithm with formal privacy guarantees, but leave
its utility analysis for
future work; see the discussion in Section~\ref{sec:conclusion-and-discussion}.

\paragraph{Private coordinate descent}

Differentially Private Coordinate Descent (DP-CD) was recently studied by
\citet{mangold2021Differentially}, who analyzed its utility and
derived corresponding lower
bounds. They showed that DP-CD can exploit
coordinate-wise regularity assumptions to use larger step-sizes, outperforming
DP-SGD when gradient coordinates are imbalanced. Our
DP-GCD also shares this property.
\citet{damaskinos2021Differentially} proposed a dual
coordinate descent algorithm for generalized linear models. Private CD
has also been used by \citet{bellet2018Personalized} in a
decentralized setting. All these works use random selection,
which fails to exploit key problem's properties such as quasi-sparsity, and
thus suffer a polynomial dependence on the dimension $p$. In contrast, our
private greedy selection rule focuses on the most
useful coordinates, thereby reducing the dependence on
$p$ to only logarithmic in such settings.

%% file: prelim.tex
\section{Preliminaries}
\label{sec:preliminaries}

In this section, we introduce important technical notions that will be used
throughout the paper.

\paragraph{Norms}
We start by defining two conjugate norms that will allow to keep track
of coordinate-wise quantities.  Let $M = \diag(M_1, \dots, M_p)$ with
$M_1, \dots, M_p > 0$, and%
\begin{align*}
  \norm{w}_{M,1} = \sum_{j=1}^p M_j^{\frac{1}{2}} \abs{w_j}\enspace,\enspace\enspace\enspace\enspace
  \norm{w}_{M^{-1}, \infty} = \max_{j \in [p]} M_j^{-\frac{1}{2}} \abs{w_j} \enspace.
\end{align*}
When $M$ is the identity matrix $I$, $\norm{\cdot}_{M,1}$ is the
standard $\ell_1$-norm and $\norm{\cdot}_{M^{-1},\infty}$ is the
$\ell_\infty$-norm.
We also define the Euclidean dot product
$\scalar{u}{v} = \sum_{j=1}^p u_i v_i$ and corresponding norms
$\norm{\cdot}_{M,2} = \scalar{\cdot}{M\cdot}^{\frac{1}{2}}$ and
$\norm{\cdot}_{M^{-1},2} =
  \scalar{\cdot}{M^{-1}\cdot}^{\frac{1}{2}}$. Similarly, we recover the
standard $\ell_2$-norm when $M=I$.

\paragraph{Regularity assumptions}
We recall classical regularity assumptions along with ones
specific to the coordinate-wise setting.
We denote by $\nabla f$ the gradient of
a differentiable function $f$, and by $\nabla_j f$ its $j$-th coordinate.
We denote by $e_j$ the $j$-th vector of $\RR^p$'s standard basis.

\textit{(Strong)-convexity.} For $q \in \{1,2\}$, a differentiable
function $f : \RR^p \rightarrow \RR$ is $\mu_{M,q}$-strongly-convex
\wrt the norm $\smash{\norm{\cdot}_{M,q}}$ if for all
$v, w \in \RR^p$,
$f(w) \ge f(v) + \scalar{\nabla f(v)}{w - v} + \frac{\mu_{M,q}}{2}\norm{w
    - v}_{M,q}^2$.  The case $M_{1,q}=\cdots=M_{p,q}=1$ recovers
standard $\mu_{I,q}$-strong convexity \wrt the $\ell_q$-norm. When
$\mu_{M,q}=0$, the function is just said to be \textit{convex}.

\textit{Component Lipschitzness.} A function
$f : \RR^p \rightarrow \RR$ is $L$-component-Lipschitz for
$L = (L_1,\dots,L_p)$ with $L_1,\dots,L_p > 0$ if for $w \in \RR^p$,
$t \in \RR$ and $j \in [p]$,
$\abs{f(w + t e_j) - f(w)} \le L_j \abs{t}$.  For $q \in \{1,2\}$, $f$
is $\Lambda_q$-Lipschitz \wrt $\norm{\cdot}_q$ if for
$v, w \in \RR^p$, $\abs{f(v) - f(w)} \le \Lambda_q \norm{v - w}_q$.

\textit{Component smoothness.} A differentiable function $f : \RR^p
  \rightarrow \RR$ is
$M$-component-smooth for $M_1,\dots,M_p > 0$ if
for $v, w \in \RR^p$,
$f(w) \le f(v) + \scalar{\nabla f(v)}{w - v} + \frac{1}{2}\norm{w - v}_{M,2}^2$.
When $M_1=\dots=M_p=\beta$, $f$ is said to be $\beta$-smooth.

Component-wise regularity assumptions are not restrictive: for
$q \in \{1,2\}$, $\Lambda_q$-Lipschitzness \wrt $\norm{\cdot}_q$
implies $(\Lambda_q, \dots, \Lambda_q)$-component-Lipschitzness and
$\beta$-smoothness implies
$(\beta, \dots, \beta)$-component-smoothness.  Yet, the actual
component-wise constants of a function can be much lower than what can
be deduced from their global counterparts. In the following of this
paper, we will use $M_{\min} = \min_{j\in[p]} M_j$,
$M_{\max} = \max_{j \in [p]} M_j$, and their Lipschitz counterparts
$L_{\min}$ and $L_{\max}$.

\begin{algorithm*}[t]
  \caption{DP-GCD: Differentially Private Greedy Coordinate Descent}
  \label{algo:private-greedy-coordinate-descent}
  \begin{algorithmic}[1]
    \State \textbf{Input:} initial $w^0 \in \RR^p$, iteration count $T > 0, \forall j \in [p],$ noise scales $\lambda_j, \lambda_j'$, step sizes $\gamma_j > 0$.
    \For{$t = 0$ to $T-1$}
    \State $\displaystyle j_t = \argmax_{j' \in [p]} \tfrac{\abs{\nabla_{j'} f(w^t) + \chi_{j'}^t}}{\sqrt{M_{j'}}}$, \hfill with $ \chi_{j'}^t \sim \Lap(\lambda'_{j'})$. \label{algo:private-greedy-coordinate-descent:choice_coord}\algorithmiccomment{Choose $j_t$ using report-noisy-max.}
    \State $\displaystyle w^{t+1} = w^t - \gamma_{j_t} (\nabla_{j_t} f(w^t) + \eta_{j_t}^t) e_{j_t}$, \hfill  with $\eta_{j_t}^t \sim \Lap(\lambda_{j_t})$. \algorithmiccomment{Update the chosen coordinate.}
    \label{algo:private-greedy-coordinate-descent:update}
    \EndFor
    \State \Return $w^T$.
  \end{algorithmic}
\end{algorithm*}

\paragraph{Differential privacy (DP)}

Let $\cD$ be a set of datasets and $\cF$ a set of possible outcomes.
Two datasets $D, D' \in \cD$ are said \textit{neighboring}
(denoted by $D \sim D'$) if they differ on at most one element.

\begin{definition}[Differential Privacy, \citealt{dwork2006Differential}]
  A randomized algorithm
  $\cA : \mathcal D
    \rightarrow \mathcal F$ is $(\epsilon, \delta)$-differentially private if,
  for all neighboring datasets $D, D' \in \mathcal D$ and all
  $S \subseteq \mathcal F$ in the range of~$\cA$:
  \begin{align*}
    \prob{\cA(D) \in S} \le e^\epsilon \cdot \prob{\cA(D') \in S} + \delta \enspace.
  \end{align*}
\end{definition}
In this paper, we consider the classic central model of DP, where a trusted
curator has access to the raw dataset and releases a model trained on this
dataset\footnote{In fact, our privacy guarantees hold even if all
  intermediate iterates are released (not just the final model).}.

A common principle for releasing a private estimate of a function
$h : \mathcal D \rightarrow \RR^p$ is to perturb it with noise.
To ensure privacy, the noise is scaled with the
sensitivity $\Delta_q(h) = \sup_{D \sim D'} \norm{h(D) - h(D')}_q$ of
$h$, with $q=1$ for Laplace, and $q=2$ for Gaussian mechanism.
In coordinate descent methods, we release coordinate-wise
gradients. The $j$-th coordinate of a loss function's gradient
$\nabla_j\ell:\RR^p\rightarrow\RR$ has sensitivity
$\Delta_1(\nabla_j f) = \Delta_2(\nabla_j f)$ ($\nabla_j f$ is a
scalar). For $L$-component-Lipschitz losses, these sensitivities are
upper bounded by $2L_j$ \citep{mangold2021Differentially}.

In our algorithm, we will also need to compute the index of the
gradient's maximal entry privately. To this end, we use the
report-noisy-argmax mechanism \citep{dwork2013Algorithmic}. This
mechanism perturbs each entry of a vector with Laplace noise,
calibrated to its \emph{coordinate-wise} sensitivities, and releases
the index of a maximal entry of this noisy vector. Revealing only this
index allows to greatly reduce the noise, in comparison to
releasing the full gradient. This will be the cornerstone of our
greedy algorithm.

%% file: greedy-cd.tex
\section{Private Greedy CD}
\label{sec:priv-greedy-coord}

In this section, we present our main contribution: the differentially
private greedy coordinate descent algorithm (DP-GCD). As described in
Section~\ref{sec:algorithm}, DP-GCD updates only one coordinate per iteration,
which is selected greedily as the (approximately) largest entry of the
gradient so as to maximize the improvement in utility at each iteration. We establish
privacy (Section~\ref{sec:privacy}) and utility (Section~\ref{sec:utility})
guarantees for DP-GCD. We further show in \Cref{sec:fast-init-conv} that
DP-GCD enjoys improved utility for high-dimensional problems with a
\emph{quasi-sparse} solution (\ie with a fraction of the parameters dominating
the others). We then provide a proximal extension of DP-GCD to non-smooth
problems (\Cref{sec:proximal-dp-gcd}) and conclude with a discussion of
DP-GCD's computational complexity in \Cref{sec:computational-cost}.

\subsection{The Algorithm}
\label{sec:algorithm}

At each iteration, DP-GCD
(\Cref{algo:private-greedy-coordinate-descent}) updates the parameter
with the greatest gradient value (rescaled by the inverse square root
of the coordinate-wise smoothness constant). This corresponds to the
Gauss-Southwell-Lipschitz rule \citep{nutini2015Coordinate}. To
guarantee privacy, this selection is done using the report-noisy-max
mechanism \citep{dwork2013Algorithmic} with noise scales $\lambda_j'$
along $j$-th entry ($j\in[p]$). DP-GCD then performs a gradient step
with step size $\gamma_j > 0$ along this direction. The gradient is
privatized using the Laplace mechanism \citep{dwork2013Algorithmic}
with scale $\lambda_j$.

\begin{remark}[Sparsity of iterates]
  \label{rem:sparse_iterates}
  When initialized at $w^0=0$, DP-GCD generates sparse iterates. Since it
  chooses its updates greedily, this gives a screening ability to the
  algorithm \citep{fang2020Greed}. We discuss the implications of this
  property in Section~\ref{sec:fast-init-conv}, where we show that DP-GCD's
  utility is improved when the problem's solution is (quasi-)sparse.
\end{remark}

\subsection{Privacy Guarantees}
\label{sec:privacy}

The privacy guarantees of DP-GCD depends on the noise scales $\lambda_j$ and
$\lambda_j'$. In
\Cref{thm:greedy-cd-privacy}, we describe how to set these values so
as to ensure that DP-GCD is
$(\epsilon,\delta)$-differentially private.

\begin{theorem}
  \label{thm:greedy-cd-privacy}
  Let $\epsilon, \delta \in (0,
    1]$. \Cref{algo:private-greedy-coordinate-descent} with
  $\lambda_j = \lambda_j' = \frac{8L_j}{n\epsilon} \sqrt{ T
      \log(1/\delta)} $ is $(\epsilon,\delta)$-DP.
\end{theorem}

\begin{sketchproof}
  (Detailed proof in \Cref{sec-app:proof-privacy})  Let
  $\epsilon' = \epsilon/\sqrt{16T\log(1/\delta)}$. At an iteration
  $t$, data is accessed twice. First, to compute the index $j_t$ of the
  coordinate to update. It is obtained as the index of the largest
  noisy entry of $f$'s gradient, with noise $\Lap(\lambda_j')$. By the
  report-noisy-argmax mechanism, $j_t$ is $\epsilon'$-DP. Second, to
  compute the gradient's $j_t$'s entry, which is released with noise
  $\Lap(\lambda_j)$.%
  The Laplace mechanism ensures that this computation is also
  $\epsilon'$-DP. \Cref{algo:private-greedy-coordinate-descent} is
  thus the $2T$-fold composition of $\epsilon'$-DP mechanisms, and the
  result follows from DP's advanced composition theorem
  \citep{dwork2013Algorithmic}.
\end{sketchproof}

\begin{remark}
  The assumption $\epsilon\in(0,1]$ is only used to give a closed-form
  expression for the noise scales $\lambda, \lambda$'s. In practice,
  we tune them numerically to obtain the desired value of $\epsilon>0$
  by the advanced composition theorem (see
  \cref{thm:greedy-cd-privacy:general-epsilon-value} in
  \Cref{sec-app:proof-privacy}), removing the assumption
  $\epsilon \le 1$.
\end{remark}

Computing the greedy update requires injecting Laplace noise that
scales with the coordinate-wise Lipschitz constants $L_1,\dots,L_p$ of
the loss. These constants are typically smaller than their global
counterpart. This allows DP-GCD to inject less noise on smaller-scaled
coordinates.

\subsection{Utility Guarantees}
\label{sec:utility}

We now prove utility upper bounds for DP-GCD.
We show that in
favorable settings (see discussion below), DP-GCD decreases the
dependence on the dimension from polynomial to logarithmic.
\Cref{thm:greedy-cd-utility} gives asymptotic utility
upper bounds, where~$\widetilde O$ ignores non-significant logarithmic
terms. Complete non-asymptotic results can be found in
\Cref{sec-app:proof-utility}.

\begin{theorem}
  \label{thm:greedy-cd-utility}
  Let $\epsilon, \delta \in (0,1]$. Assume $\ell(\cdot; d)$ is a
  convex and $L$-component-Lipschitz loss function for all
  $d \in \cX$, and $f$ is $M$-component-smooth. Define $\cW^*$ the set
  of minimizers of $f$, and $f^*$ the minimum of $f$. Let
  $w_{priv}\in\mathbb{R}^p$ be the output of
  \Cref{algo:private-greedy-coordinate-descent} with step sizes
  $\gamma_j = {1}/{M_j}$, and noise scales
  $\lambda_1,\dots,\lambda_p$, $\lambda'_1,\dots,\lambda'_p$ set as in
  Theorem~\ref{thm:greedy-cd-privacy} (with $T$ chosen below) to
  ensure $(\epsilon,\delta)$-DP. Then, the following holds for any
  $\zeta\in(0,1]$:
  \begin{enumerate}[leftmargin=12pt]
    \item When $f$ is convex, we define the quantity
          $R_{M,1}=\max_{\scriptstyle w\in\RR^p~}\!\!\max_{\scriptstyle w^*\in
              \cW^*}\! \left\{ \norm{w-w^*}_{M,1} \!\mid\! f(w) \le f(w^0)
            \right\}$. Assume the initial optimality gap is
          $f(w^0) - f^* \ge 16 L_{\max} \sqrt{T\log(1/\delta)
              \log(2Tp/\zeta)}/M_{\min} n \epsilon$, and set
          $T = O(n^{2/3} \epsilon^{2/3} R_{M,1}^{2/3} M_{\min}^{1/3} /
            L_{\max}^{2/3} \log(1/\delta)^{1/3})$. Then with probability at
          least $1-\zeta$,
          \begin{align*}
            f(w_{priv}) - f^*
            = \widetilde O\bigg(\frac{R_{M,1}^{4/3} L_{\max}^{2/3} \log(1/\delta) \log(p/\zeta)}{n^{2/3}\epsilon^{2/3}M_{\min}^{1/3}}\bigg)\enspace.
          \end{align*}
    \item When $f$ is $\mu_{M,1}$-strongly convex
          w.r.t. $\smash{\norm{\cdot}_{M,1}}$, set
          $T = O\left(\frac{1}{\mu_{M,1}} \log(\frac{M_{\min}
                  \mu_{M,1} n \epsilon (f(w^0)-f(w^*))}{L_{\max}
                  \log(1/\delta) \log(2p/\zeta)})\right)$. Then with
          probability at least $1-\zeta$,
          \begin{align*}
            f(w_{priv}) - f^*
            = \widetilde O\bigg(
            \frac{L_{\max}^2 \log(1/\delta) \log({2p}/{\mu_M\zeta})}{M_{\min} \mu_{M,1}^2 n^2 \epsilon^2}
            \bigg)\enspace.
          \end{align*}
  \end{enumerate}%
\end{theorem}
\begin{sketchproof}(Detailed proof in \Cref{sec-app:proof-utility}). %
  We start by proving a noisy ``descent lemma''.  Since $f$ is smooth,
  we have
  $f(w^{t+1}) \le f(w^t) - \frac{1}{2M_j} \nabla_j f(w^t)^2 +
  \frac{1}{2M_j} (\eta_j^t)^2$. The greedy selection of $j$ gives that
  $- \tfrac{1}{M_j}(\nabla_j f(w^t) + \chi_j)^2 \le - \norm{\nabla
    f(w^t) + \chi}_{M^{-1},\infty}^2$. We then use the inequality
  $(a+b)^2 \le 2a^2 + 2b^2$ for $a,b\in\RR$, and convexity arguments
  to prove the lemma. When $f$ is convex, we have%
  \begin{align*}
     & f(w^{t+1}) - f(w^*)
    \le f(w^t) - f(w^*)
    - \frac{(f(w^t) - f(w^*))^2}{8 \norm{w^t-w^*}_{M,1}^2}
    + \frac{\abs{\eta_j^t}^2}{2M_j}
    + \frac{ \abs{\chi^t_{j}}^2 }{2M_{j}}
    + \frac{ \abs{\chi^t_{j^*}}^2 }{4M_{j^*}} \enspace.
  \end{align*}
  There, we observe that, at each iteration, either (i) $w^t$ is far
  enough from the optimum, and the value of the objective decreases
  with high probability, either (ii) $w^t$ is close to the optimum,
  then all future iterates remain in a ball whose radius depends on
  the scale of the noise. We prove this key property rigorously in
  \Cref{sec:key-lemma-behavior}.

  When $f$ is $\mu_{M,1}$-strongly-convex \wrt $\norm{\cdot}_{M,1}$,
  we obtain
  \begin{align*}
    f(w^{t+1}) - f(w^*)
     & \le \Big( 1 - \frac{\mu_{M,1}}{4} \Big)
    (f(w^t) - f(w^*))
       + \frac{\abs{\eta_j^t}^2}{2M_j}
    + \frac{ \abs{\chi^t_{j}}^2 }{2M_{j}}
    + \frac{ \abs{\chi^t_{j^*}}^2 }{4M_{j^*}} \enspace,
  \end{align*}
  and the result follows by induction. In both settings, we use
  Chernoff bounds to obtain a high-probability result.
\end{sketchproof}

\begin{remark}
  The lower bound on
  $f(w^0) - f^*$ in Theorem~\ref{thm:greedy-cd-utility}
  is a standard assumption in the analysis
  of inexact coordinate descent methods: it ensures that sufficient
  decrease is possible despite the noise. A similar assumption is made
  by \citet{tappenden2016Inexact}, see Theorem 5.1 therein.
\end{remark}

\paragraph{Discussion of the utility bounds}
One of the key properties of DP-GCD is that its utility is dictated by
$\ell_1$-norm quantities ($R_{M,1}$ and $\mu_{M,1}$). Remarkably, this
arises without enforcing any $\ell_1$ constraint in the problem, which is in
stark contrast with
private Frank-Wolfe algorithms (DP-FW) that
require such constraints
\citep{talwar2015Nearly,asi2021Private,bassily2021NonEuclidean}. To better
grasp the implications of this, we discuss our results in two
regimes considered in previous work (see Section~\ref{sec:related-works}): (i) when these
$\ell_1$-norm quantities are bounded (similarly to DP-FW algorithms),
and (ii) when their
$\ell_2$-norm counterparts are bounded (similarly to DP-SGD-style
algorithms).

\textit{Bounded in $\ell_1$-norm.} When $R_{M,1}$ and $\mu_{M,1}$ are
$O(1)$, as assumed in prior work on DP-FW
\citep{talwar2015Nearly,asi2021Private,bassily2021NonEuclidean},
DP-GCD's dependence on the dimension is \textit{logarithmic}. For
convex objectives, its utility is $O(\log(p)/n^{2/3}\epsilon^{2/3})$,
matching that of DP-FW and known lower bounds
\citep{talwar2015Nearly}. For strongly-convex problems,
DP-GCD is the first algorithm to achieve a $O(\log(p)/n^2\epsilon^2)$
utility. Indeed, the only competing result in this setting, due to
\citet{asi2021Private}, obtains a worse
 utility of $O(\log(p)^{4/3}/n^{4/3}\epsilon^{4/3})$ by using an
impractical reduction of DP-FW to the convex case. DP-GCD outperforms
this prior result without suffering the extra complexity due to the
reduction.

\textit{Bounded in $\ell_2$-norm.} Consider $R_{M,2}$ and
$\mu_{M,2}$, the $\ell_2$-norm counterparts of $R_{M,1}$ and
$\mu_{M,1}$. Assume that $R_{M,2}$ and
$\mu_{M,2}$ are both $O(1)$, as considered in DP-SGD and its variants \citep{bassily2014Differentially,wang2017Differentially}.
We compare these quantities
using the following inequalities
\citep[see][]{stich2017Approximate,nutini2015Coordinate}:
\begin{align*}
  R_{M,2}^2 \le R_{M,1}^2 \le p R_{M,2}^2\enspace,\enspace\enspace\enspace\enspace
  \tfrac{1}{p}\mu_{M,2} \le \mu_{M,1} \le
  \mu_{M,2}\enspace.
\end{align*}
In the best case of these inequalities, the $O(\log p)$ utility bounds
of the bounded $\ell_1$ norm regime are preserved in the bounded
$\ell_2$ scenario. In the worst case, the utility of DP-GCD becomes
$\widetilde O (p^{2/3}/n^ {2/3}\epsilon^{2/3})$ and
$\widetilde O(p^2/n^2\epsilon^2)$ for convex and strongly-convex
objectives respectively. These worst-case results match DP-FW's
utility in the convex setting (see \eg
\Citet{asi2021Private}), but they do not match DP-SGD's utility.
However, this sheds light on an interesting phenomenon: DP-GCD
\emph{interpolates between $\ell_1$- and $\ell_2$-norm
  regimes}. Indeed, it lies somewhere between the two extreme cases we
just described, depending on how the $\ell_1$- and $\ell_2$-norm
constants compare. Most interestingly, it does so without \textit{a
  priori} knowledge of the problem or explicit constraint on the
domain. Whether there exists an algorithm that yields optimal utility
in all regimes is an interesting open question.

\paragraph{Coordinate-wise regularity} Due to its use of
coordinate-wise step sizes, DP-GCD can adapt to coordinate-wise
imbalance of the objective in the same way as its randomized counterpart,
DP-CD, where coordinates are chosen uniformly at random
\citep{mangold2021Differentially}. This adaptivity notably appears in \Cref
{thm:greedy-cd-utility}
through the measurement of $R_{M,1}$ and $\mu_{M,1}$ relatively to the
scaled norm $\norm{\cdot}_{M,1}$ (as defined in
\Cref{sec:preliminaries}). We refer to
\citep{mangold2021Differentially} for detailed discussion of these
quantities and the associated gains compared to full gradient methods like
DP-SGD.

\subsection{Better Utility on Quasi-Sparse Problems}
\label{sec:fast-init-conv}

In addition to the general utility results presented above, we now exhibit a
specific setting where DP-GCD performs especially
well, namely strongly-convex problems whose solutions
are dominated by a few parameters. We call such vectors quasi-sparse.

\begin{definition}[$(\alpha,\tau)$-quasi-sparsity]
  \label{gcd:quasi-sparsity}
  A vector $w\in\RR^p$ is $(\alpha,\tau)$-quasi-sparse if it has at
  most $\tau$ entries superior to $\alpha$ (in modulus). When
  $\alpha=0$, the vector is called $\tau$-sparse.
\end{definition}
Note that any vector in $\RR^p$ is $
(0,p)$-quasi-sparse, and for any $\tau$ there exists $\alpha>0$ such that the
vector is $(\alpha,\tau)$-quasi-sparse. In fact, $\alpha$ and $\tau$ are linked, and
$\tau(\alpha)$ can be seen as a function of $\alpha$. Of course,
quasi-sparsity will only yield meaningful improvements when $\alpha$
and $\tau$ are small simultaneously.

We now state the main result of
this section, which shows that DP-GCD
(initialized with $w^0=0$) converges to a good approximate solution in few
iterations for problems with quasi-sparse solutions.

\begin{theorem}[Proof in \Cref{sec:fast-init-conv-sup}]
  \label{gcd:fast-initial-convergence}
  Consider $f$ satisfying the hypotheses of \Cref{thm:greedy-cd-utility}, with
  \Cref{algo:private-greedy-coordinate-descent} initialized at
  $w^0=0$. We denote its output $w^T$, and assume that its iterates
  remain $s$-sparse for some $s \le p$.  Assume that $f$ is
  $\mu_{M,2}$-strongly-convex \wrt $\norm{\cdot}_{M,2}$, and that the
  (unique) solution of problem~\eqref{eq:dp-erm} is
  $(\alpha,\tau)$-quasi-sparse for some $\alpha,\tau\ge 0$.  Let
  $0\le T \le p-\tau$ and $\zeta\in[0,1]$. Then with probability at
  least $1 - \zeta$:
  \begin{align*}
    f(w^T) - f^*
     & \le \prod_{t=1}^T \Big(1 - \frac{\mu_{M,2}}{4(\tau+\min(t,s))}\Big) (f(w^0) - f^*)
       + \widetilde O \Big((T+\tau)(p - \tau)\alpha^2
    + \frac{L_{\max}^2T(T+\tau)}{M_{\min} \mu_{M,2} n^2 \epsilon^2}\Big)
    \enspace.
  \end{align*}
  Setting
  $T = \frac{s+\tau}{\mu_{M,2}} \log((f(w^0)-f^*)
  M_{\min}\mu_{M,2}n^2\epsilon^2/L^2)$, and assuming
  $\alpha^2 = O\left( L_{\max}^2 (s+\tau) / M_{\min}\mu_{M,2}^2
    pn^2\epsilon^2 \right)$, we obtain that with probability at least
  $1-\zeta$,
  \begin{align*}
    f(w^T) - f^* = \widetilde O\left( \frac{L_{\max}^2}{M_{\min}}
    \frac{(s+\tau)^2 \log(2p/\zeta)}{\mu_{M,2} n^2\epsilon^2 } \right)
    \enspace.
  \end{align*}
\end{theorem}
Here, strong convexity is measured in $\ell_2$ norm but the dependence
on the dimension is reduced from $p$, the ambient space dimension, to
$(s+\tau)^2$, the \emph{effective dimension of the space where the
  optimization actually takes place}. For high-dimensional sparse
problems, the latter is typically much smaller and yields a large
improvement in utility. Note that it is not necessary for the solution
to be perfectly sparse: it suffices that most of its mass is
concentrated in a fraction of the coordinates. Notably, when
$\alpha^2 = O(L_{\max}^2T/M_{\min}\mu_{M,2}pn^2\epsilon^2)$, the lack
of sparsity is smaller than the noise, and does not affect the
rate. It generalizes the results by \citet{fang2020Greed} for
non-private and sparse settings, that we recover when~$\alpha=0$
and~$\epsilon\rightarrow +\infty$.

In practice, the assumption over the iterates' sparsity is often met with
$s\ll
p$. In
the non-private setting, greedy coordinate descent is known to focus on
coordinates that are non-zero in the solution
\citep{Massias_Gramfort_Salmon17b}: this keeps iterates' sparsity
close to the one of the solution. Furthermore, due to privacy constraints,
DP-GCD will often run for $T \ll p$ iterations. This is especially true in
high-dimensional problems, where the amount of noise required to
guarantee privacy does not allow many iterations (\lcf experiments in
\Cref{sec:experiments-1}).

\subsection{Proximal DP-GCD}
\label{sec:proximal-dp-gcd}

In \Cref{sec:fast-init-conv}, we proved that DP-GCD's utility
is improved when problem's solution is (quasi-)sparse. This motivates
us to consider problems with sparsity-inducing regularization, such as
the $\ell_1$ norm of $w$ \citep{tibshirani1996Regression}.  To tackle
such non-smooth terms, we propose a proximal version of DP-GCD (for
which the same privacy guarantees hold), building upon the multiple
greedy rules that have been proposed for the nonsmooth setting
\citep[see \eg][]{tseng2009Coordinate,nutini2015Coordinate}. We
describe this extension in \Cref{sec:greedy-coord-desc}, and study it
numerically in \Cref{sec:experiments-1}.

\subsection{Computational Cost}
\label{sec:computational-cost}

Each iteration of DP-GCD requires computing a full gradient, but
only uses one of its coordinates. In non-private optimization, one would
generally be better off performing the full update to avoid wasting
computation. This is not
the case when gradients are private. Indeed, using the full gradient
requires privatizing $p$ coordinates, even when only a few of
them may be needed. Conversely, the report noisy max
mechanism
\citep{dwork2013Algorithmic} allows to select these entries
\emph{without paying the full privacy cost of dimension}. Hence, the
greedy updates of DP-GCD reduce the noise needed at the cost of more
computation.

In practice, the higher computational cost of each iteration may not always
translate in a significantly larger cost overall: as
shown by our theoretical results, DP-GCD is able to exploit the
\emph{quasi-sparsity} of the solution to progress fast and only a
handful of iterations may be needed to reach a good private
solution. In contrast, most updates of classic private optimization algorithms
(like DP-SGD) may not be worth doing, and lead to unnecessary injection
of noise. We illustrate this phenomenon numerically in
\Cref{sec:experiments-1}.

%% file: expe.tex
\section{Experiments}
\label{sec:experiments-1}

\looseness=-1 In this section, we evaluate the practical performance
of DP-GCD on linear models using the logistic and squared loss with
$\ell_1$ and $\ell_2$ regularization. We compare DP-GCD to two
competitors: differentially private stochastic gradient descent
(DP-SGD) with batch size~$1$
\citep{bassily2014Differentially,abadi2016Deep}, and differentially
private randomized coordinate descent (DP-CD)
\citep{mangold2021Differentially}.  The code is available
online\footnote{\url{https://gitlab.inria.fr/pmangold1/greedy-coordinate-descent}}
and in the supplementary.

\paragraph{Datasets}
The first two datasets, coined \texttt{log1} and \texttt{log2}, are
synthetic.  We generate a design matrix $X\in\RR^{1,000 \times 100}$
with unit-variance, normally-distributed columns. Labels are computed
as $y=Xw^{(true)}+\varepsilon$, where $\varepsilon$ is
normally-distributed noise and $w^{(true)}$ is drawn from a log-normal
distribution of parameters $\mu=0$ and $\sigma=1$ or $2$ %
respectively. This makes $w^{(true)}$ quasi-sparse.
The
\sparse dataset is generated similarly, with $X\in\RR^{1,000\times 1,000}$ and $w^{(true)}$
having only $10$ non-zero values. The \texttt{california} dataset can
be downloaded from \texttt{scikit-learn}
\citep{pedregosa2011Scikitlearn} while \texttt{mtp}, \texttt{madelon}
and \texttt{dorothea} are available in the \texttt{OpenML} repository
\citep{vanschoren2014OpenML}; see summary in
\Cref{tab:datasets-desc}. We discuss the levels of (quasi)-sparsity of
each problem's solution in \Cref{sec:experimental-details}.

\paragraph{Algorithmic setup}
\textit{(Privacy.)} For each
algorithm, the tightest noise scales are computed numerically to
guarantee a suitable privacy level of $(1,1/n^2)$-DP, where $n$ is the number
of records in the dataset.
For DP-CD and DP-SGD, we privatize the gradients with the Gaussian
mechanism \citep{dwork2013Algorithmic}, and account for privacy
tightly using Rényi differential privacy (RDP)
\citep{mironov2017Renyi}. For DP-SGD, we use RDP amplification for the
subsampled Gaussian mechanism \citep{mironov2019Renyi}.

\textit{(Hyperparameters.)} For DP-SGD, we use constant step sizes and standard
gradient clipping \citep{abadi2016Deep}. For DP-GCD and DP-CD, we set
the step sizes to $\eta_j=\tfrac{\gamma}{M_j}$, and adapt the coordinate-wise
clipping thresholds from one hyperparameter, as proposed by
\citet{mangold2021Differentially}.
For each algorithm, we thus tune
two hyperparameters: one step-size and one clipping threshold; see also
\Cref{sec:experimental-details}.

\textit{(Plots.)} In all experiments, we plot the relative error to
the \emph{non-private} optimal objective value for the best set of
hyperparameters (averaged over $5$ runs), as a
function of the number of passes on the data. Each pass
corresponds to $p$ iterations of DP-CD, $n$ iterations of DP-SGD and $1$
iteration of DP-GCD. This guarantees the same amount of
computation for each algorithm, for each x-axis tick.

\begin{table*}
  \centering
  \caption{Number of records and features in each dataset.}
  \label{tab:datasets-desc}
  {%
    \begin{tabular}{ccccccc}
      \toprule
               & \texttt{log1}, \texttt{log2} & \sparse & \texttt{mtp} & \texttt{dorothea} & \texttt{california} & \texttt{madelon} \\
      \midrule
      Records  & $1,000$                      & $1,000$         & $4,450$      & $800$           & $20,640$            & $2,600$          \\
      Features & $100$                        & $1,000$         & $202$        & $88,119$        & $8$                 & $501$            \\
      \bottomrule
    \end{tabular}
  }
\end{table*}

\begin{figure*}[t]
  \captionsetup[subfigure]{justification=centering}
  \centering
  \begin{subfigure}{0.04\linewidth}
    \centering
    \vspace{-1.5em}
    \includegraphics[width=\linewidth]{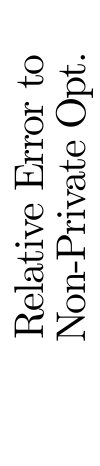}
    \begin{minipage}{.1cm}
      \vfill
    \end{minipage}
  \end{subfigure}%
  \begin{subfigure}{0.235\linewidth}
    \centering
    \includegraphics[width=\linewidth]{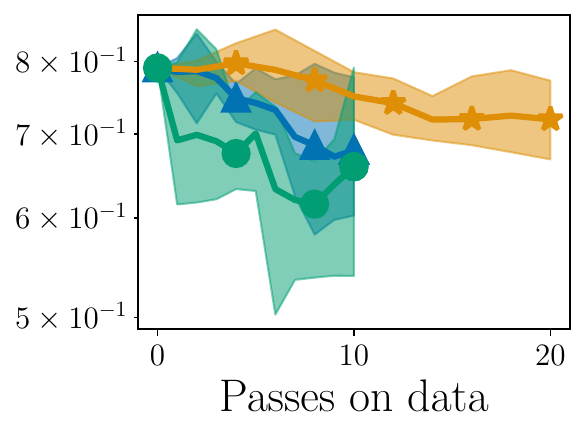}
    \caption{\texttt{log1} \\ Logistic + L2 ($\lambda=1\text{e-}3$)}
    \label{fig:expe-skewed-1}
  \end{subfigure}%
  \begin{subfigure}{0.235\linewidth}
    \centering
    \includegraphics[width=\linewidth]{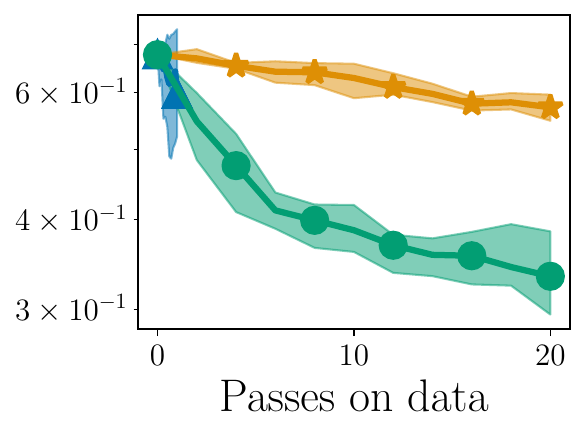}
    \caption{\texttt{log2} \\ Logistic + L2 ($\lambda=1\text{e-}3$)}
    \label{fig:expe-skewed-2}
  \end{subfigure}%
  \begin{subfigure}{0.235\linewidth}
    \centering
    \includegraphics[width=\linewidth]{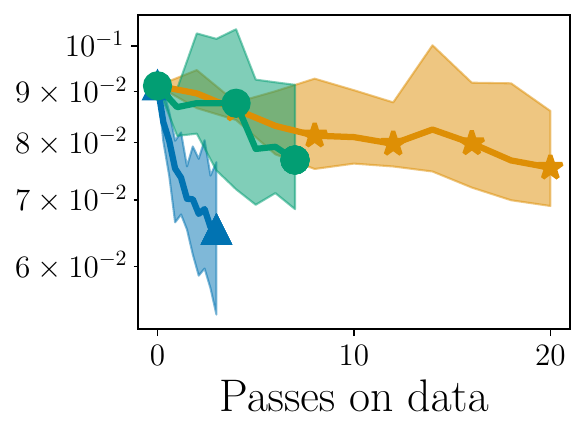}
    \caption{\texttt{mtp}\\ LS + L2 ($\lambda=5\text{e-}8$)}
    \label{fig:expe-mtp}
  \end{subfigure}%
  \begin{subfigure}{0.235\linewidth}
    \centering
    \includegraphics[width=\linewidth]{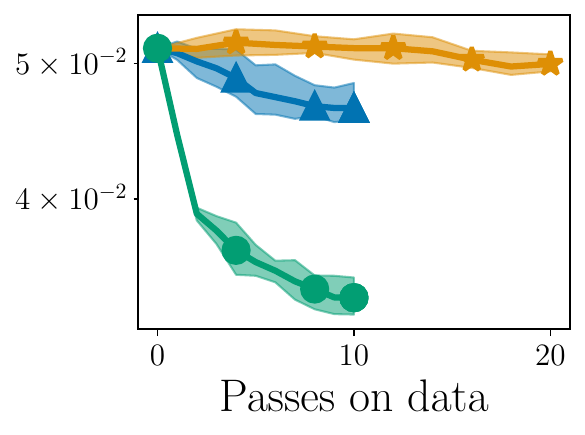}
    \caption{\texttt{madelon} \\ Logistic + L2 ($\lambda=1$)}
    \label{fig:expe-madelon-l2}
  \end{subfigure}%

  \captionsetup[subfigure]{justification=centering}
  \centering
  \begin{subfigure}{0.04\linewidth}
    \centering
    \vspace{-1.5em}
    \includegraphics[width=\linewidth]{plots/xlegend.pdf}
    \begin{minipage}{.1cm}
      \vfill
    \end{minipage}
  \end{subfigure}%
  \begin{subfigure}{0.235\linewidth}
    \centering
    \includegraphics[width=\linewidth]{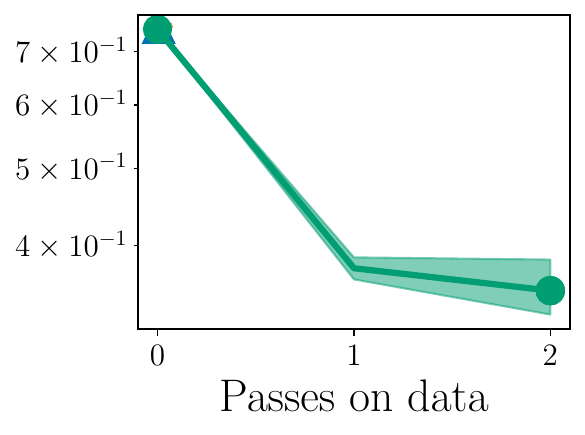}
    \caption{\sparse \\ LASSO ($\lambda=30$)}
    \label{fig:expe-sparse-lasso}
  \end{subfigure}%
  \vspace{-0.1cm}
  \begin{subfigure}{0.235\linewidth}
    \centering
    \includegraphics[width=\linewidth]{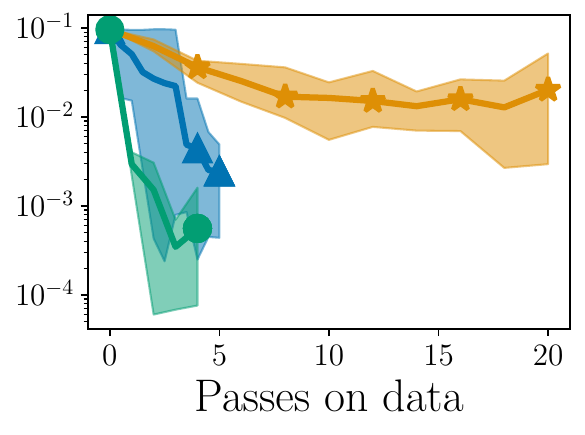}
    \caption{\texttt{california}\\ LASSO  ($\lambda=0.1$)}
    \label{fig:expe-california}
  \end{subfigure}%
  \begin{subfigure}{0.235\linewidth}
    \centering
    \includegraphics[width=\linewidth]{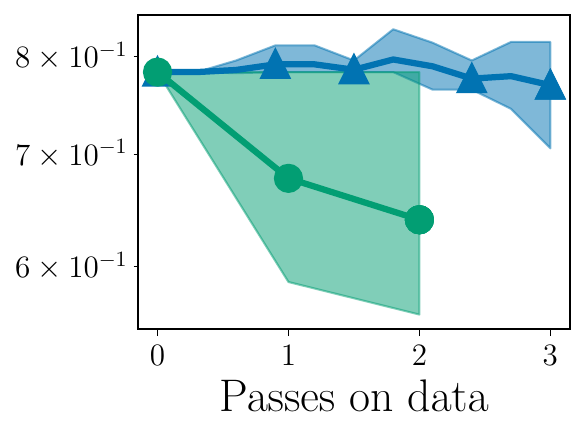}
    \caption{\texttt{dorothea}\\ Logistic + L1  ($\lambda=0.01$)}
    \label{fig:expe-dorothea}
  \end{subfigure}%
  \begin{subfigure}{0.235\linewidth}
    \centering
    \includegraphics[width=\linewidth]{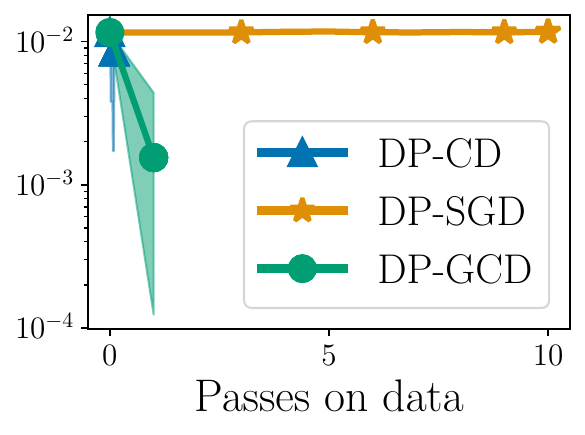}
    \caption{\texttt{madelon}\\ Logistic + L1  ($\lambda=0.05$)}
    \label{fig:expe-madelon-l1}
  \end{subfigure}%

  \caption{\looseness=-1 Relative error (min/mean/max over 5~runs) to
  non-private
  optimal for DP-GCD (our approach) versus DP-CD and
    DP-SGD. On the x-axis, $1$ tick represents
    a full access to the data: $p$ iterations of DP-CD, $n$ iterations
    of DP-SGD and $1$ iteration of DP-GCD. Number of iterations,
    clipping thresholds and step sizes are tuned simultaneously for
    each algorithm.}
  \label{fig:expe-nice}
\end{figure*}

\paragraph{DP-GCD exploits problem structure}
In the higher-dimensional datasets \sparse and \dorothea, where
$p \ge n$, DP-GCD is the only algorithm that manages to do multiple
iterations and to decrease the objective value (see
\Cref{fig:expe-sparse-lasso,fig:expe-dorothea}). In both problems,
solutions are sparse due to the $\ell_1$ regularization. This shows
that DP-GCD's greedy selection of updates can exploit this property
to find relevant non-zero coefficients (see
\Cref{tab:finding-of-support} in \Cref{sec:experimental-details}),
even when this selection is noisy.  The lower-dimensional datasets
\logg, \loggg and \madelon (where $p < n$) are still too high
dimensional (relatively to $n$) for DP-SGD and DP-CD to make
significant progress. In contrast, DP-GCD exploits the fact that
solutions are quasi-sparse to find good approximate solutions quickly
(see
\Cref{fig:expe-skewed-1,fig:expe-skewed-2,fig:expe-madelon-l2,fig:expe-sparse-lasso,fig:expe-dorothea,fig:expe-madelon-l1}).
On the low-dimensional dataset \california, DP-GCD is roughly on par
with DP-SGD and DP-CD (see \Cref{fig:expe-california}). This is due to
the additional noise term introduced by the greedy selection rule: in
such setting, the lower number of iterations does not compensate for
this as much as in higher-dimensional problems. A similar phenomenon
arise in \mtp (\Cref{fig:expe-mtp}), whose solution is not imbalanced
enough for DP-GCD to be superior to its competitors.

\paragraph{Computational complexity}
As discussed in \Cref{sec:computational-cost}, one iteration of
DP-GCD requires a full pass on the data. This is as costly as
$p$ iterations of DP-CD or $n$ iterations of DP-SGD. Nonetheless, on
many problems, DP-GCD requires just as many passes on the data as
DP-CD and DP-SGD
(\Cref
{fig:expe-skewed-1,fig:expe-mtp,fig:expe-madelon-l2,fig:expe-sparse-lasso,fig:expe-california}). When more computation is required, it also provides significantly better solutions than DP-CD and DP-SGD
(\Cref{fig:expe-skewed-2}). %
This is in line with our theoretical results from \Cref{sec:fast-init-conv}.

%% file: conclu.tex
\section{Conclusion and Discussion}
\label{sec:conclusion-and-discussion}

We proposed DP-GCD, a greedy coordinate descent algorithm for
DP-ERM. In favorable settings, DP-GCD achieves utility guarantees of
$O(\log(p)/n^{2/3}\epsilon^{2/3})$ and $O(\log(p)/n^2\epsilon^2)$ for
convex and strongly-convex objectives. It is the first algorithm to
achieve such rates without solving an $\ell_1$-constrained
problem. Instead, we show that DP-GCD depends on $\ell_1$-norm
quantities and automatically adapts to the structure of the
problem. Specifically, DP-GCD interpolates between logarithmic
and polynomial dependence on the dimension, depending on the
problem. Thus, DP-GCD constitutes a step towards the design of an
algorithm that adjusts to the appropriate $\ell_p$ structure of a
problem \citep[see][]{bassily2021NonEuclidean,asi2021Private}.

We also showed that DP-GCD adapts to the quasi-sparsity of the
problem, without requiring \textit{a priori} knowledge about it. In such
problems, it converges to a good approximate
solution in few iterations. This improves utility, and reduces the
polynomial dependence on the dimension to a polynomial dependence on
the (much smaller) quasi-sparsity level of the solution.

We also proposed and evaluated a proximal variant of DP-GCD, allowing
non-smooth, sparsity-inducing regularization. While it is not covered by our
utility guarantees, we note that the only existing analysis of such variants
in the non-private setting is the one of \citet{karimireddy2019Efficient}
for $\ell_1$ and box constraints. Their proof relies on an
alternation between \good (that provably progress) and \bad steps
(that do not increase the objective), which does not transfer to the private
setting. Extending such results to DP-ERM is an exciting
direction for future work.

%% file: sup-proof-privacy.tex
\section{Proof of Privacy}
\label{sec-app:proof-privacy}

\begin{restate-theorem}{\ref{thm:greedy-cd-privacy}}
  Let $\epsilon, \delta \in (0,
    1]$. \Cref{algo:private-greedy-coordinate-descent} with
  $\lambda_j = \lambda_j' = \frac{8L_j}{n\epsilon} \sqrt{ T
      \log(1/\delta)} $ is $(\epsilon,\delta)$-DP.
\end{restate-theorem}

\begin{proof}
  In each iteration of \Cref{algo:private-greedy-coordinate-descent},
  the data is accessed twice: once to choose the coordinate and once
  to compute the private gradient. In total, data is thus queried $2T$
  times.

  Let $\lambda_j = \lambda_j' = \frac{2L_j}{n\epsilon'}$.  For
  $j\in[p]$, the gradient's $j$-th entry has sensitivity $2L_j$. Thus,
  by the report noisy max mechanism \citep{dwork2013Algorithmic}, the
  greedy choice of $j$ is $\epsilon'$-DP. By the Laplace mechanism
  \citep{dwork2013Algorithmic}, computing the corresponding gradient
  coordinate is also $\epsilon'$-DP.

  The advanced composition theorem for differential privacy thus
  ensures that the $2T$-fold composition of these mechanisms is
  $(\epsilon, \delta)$-DP for $\delta>0$ and
  \begin{align}
    \label{thm:greedy-cd-privacy:general-epsilon-value}
    \epsilon=\sqrt{4T\log(1/\delta)}\epsilon' + 2T\epsilon'
  (\exp(\epsilon') - 1)\enspace,
  \end{align}
  where we recall that
  $\epsilon' = \frac{2L_j}{n\lambda_j} = \frac{2L_j}{n\lambda_j'}$ for
  all $j\in[p]$. When $\epsilon \le 1$, we can give a simpler
  expression \citep[see Corollary 3.21 of][]{dwork2013Algorithmic}:
  with $\epsilon' = {\epsilon}/{4\sqrt{T\log(1/\delta)}}$,
  \Cref{algo:private-greedy-coordinate-descent} is
  $(\epsilon,\delta)$-DP for
  $\lambda_j = \lambda_j' = 8L_j\sqrt{T\log(1/\delta)}/{n\epsilon}$.
\end{proof}

%% file: sup-proof-utility.tex
\section{Proof of Utility}
\label{sec-app:proof-utility}

In this section, we prove Theorem~\ref{thm:greedy-cd-utility} and
Theorem~\ref{gcd:fast-initial-convergence}, giving
utility upper bounds for DP-GCD. We obtain these high-probability
results through a careful examination of the properties of DP-GCD's
iterates, and obtain high-probability results by using concentration
inequalities (see \Cref{sec:technical-lemmas}).

In \Cref{sec:descent-lemma}, we prove a general descent lemma, which
implies that iterates of DP-GCD converge (with high probability) to a
neighborhood of the optimum. This property is proven rigorously in
\Cref{sec:key-lemma-behavior}, and we give the utility results for
general convex functions in \Cref{sec:conv-util-result}. Under the
additional assumption that the objective is strongly convex, we prove
better utility bounds in \Cref{sec:util-strongly-conv}. These bounds
follow from a key lemma (see \Cref{sec:key-ineq-strongly}), which
implies linear convergence to a neighborhood of the optimum. We then
use this result in two settings, obtaining two different rates: first
in a general setting (in \Cref{sec:gener-util-result}), then under the
additional assumption that the problem's solution is quasi-sparse (in
\Cref{sec:fast-init-conv-sup}).

\subsection{Concentration Lemma}
\label{sec:technical-lemmas}

To prove high-probability utility results, we first bound (in
\Cref{dp-gcd:techical-lemma-chernoff}) the probability for a sum of
squared Laplacian variables to exceed a given threshold.

\begin{lemma}
  \label{dp-gcd:techical-lemma-chernoff}
  Let $K>0$ and $\lambda_1, \dots, \lambda_K > 0$. Define
  $X_k \sim \Lap(\lambda_k)$ and
  $\lambda_{\max}=\max_{k\in[K]} \lambda_k$. For any $\beta > 0$, it holds that
  \begin{align}
    \prob{\sum_{k=1}^K X_k^2 \ge \beta}
    \le 2^K  \exp\left( - \frac{\sqrt{\beta}}{2\lambda_{\max}} \right)
    \enspace.
  \end{align}
\end{lemma}

\begin{proof}
  We first remark that
  $(\sum_{k=1}^K \abs{X_k})^2 = \sum_{k=1}^K \sum_{k'=1}^K
  \abs{X_k}\abs{X_{k'}} \ge \sum_{k=1}^K X_k^2$. Therefore
  \begin{align}
    \prob{\sum_{k=1}^K X_k^2 \ge a^2}
    & \le \prob{\Big(\sum_{k=1}^K \abs{X_k}\Big)^2 \ge a^2}
      = \prob{\Big(\sum_{k=1}^K \abs{X_k}\Big) \ge a}
      \enspace.
  \end{align}
  Chernoff's inequality now gives, for any $\gamma > 0$,
  \begin{align}
    \prob{\sum_{k=1}^K \abs{X_k} \ge a}
    & \le \exp(-\gamma a) \expec{}{ \exp(\gamma \sum_{k=1}^K\abs{X_k})}
      \label{dp-gcd:techical-lemma-chernoff:full-ineq}
      \enspace.
  \end{align}
  By the properties of the exponential and the $X_k$'s independence,
  we can rewrite the inequality as
  \begin{align}
    \prob{\sum_{k=1}^K \abs{X_k} \ge a}
    \le \exp(-\gamma a) \expec{}{ \prod_{k=1}^K  \exp\Big(\gamma \abs{X_k}\Big)}
    = \exp(-\gamma a) \prod_{k=1}^K\expec{}{\exp\Big(\gamma \abs{X_k}\Big)}
    \enspace.
  \end{align}
  We can now compute the expectation of $\exp(\gamma\abs{X_k})$ for $k\in[K]$,
  \begin{align}
    \expec{}{\exp\Big(\gamma \abs{X_k}\Big)}
    & = \frac{1}{2\lambda_k}
      \int_{-\infty}^{+\infty}\! \exp(\gamma \abs{x}) \exp(-\frac{\abs{x}}{\lambda_k}) dx
     = \frac{1}{\lambda_k}
      \int_{0}^{+\infty}\! \exp\Big((\gamma-\frac{1}{\lambda_k}) x\Big) dx
      \enspace.
  \end{align}
  We choose $\gamma = 1/2\lambda_{\max}$, such that
  $\gamma \le 1/2\lambda_k$ for all $k \in [K]$ and obtain
  \begin{align}
    \expec{}{\exp\Big(\gamma \abs{X_k}\Big)}
    & = \frac{1}{\lambda_k} \frac{1}{\frac{1}{\lambda_k} - \gamma}
      = \frac{1}{1-\gamma \lambda_k}
      \le 2
      \enspace.
  \end{align}
  Plugging everything together, we have proved that
  \begin{align}
    \prob{\sum_{k=1}^K X_k^2 \ge a^2}
    \le \prob{\sum_{k=1}^K \abs{X_k} \ge a}
    \le 2^K \exp(-\frac{a}{2\lambda_{\max}})
      \enspace,
  \end{align}
  and taking $a = \sqrt{\beta}$ gives the result.
\end{proof}

\subsection{Descent Lemma}
\label{sec:descent-lemma}

We now prove a noisy descent lemma for DP-GCD
(\Cref{gcd-utility-proof:descent-lemma}). This lemma bounds the
suboptimality $f(w^{t+1})-f(w^*)$ at time $t+1$ as a function of the
suboptimality $f(w^t) - f(w^*)$ at time $t$, of the gradient's largest
entry and of the noise. At this point, we remark that when the
gradient is large enough, it is very probable that
$\frac{1}{8} \norm{\nabla f(w^t)}_{M^{-1},\infty}^2 \ge
\frac{1}{2M_{j}} \abs{\eta_{j}^t}^2 + \frac{1}{2M_{j}}
\abs{\chi_{j}^t}^2 + \frac{1}{4M_{j^*}} \abs{\chi_{j^*}^t}^2$: this
implies that the value of the objective function decreases with high
probability, even under the presence of noise. This observation will
be crucial for proving utility for general convex functions.

\begin{lemma}
  Let $t\ge 0$ and $w^t,w^{t+1}\in\RR^p$ two consecutive iterates of
  \Cref{algo:private-greedy-coordinate-descent}, with $\gamma_j=1/M_j$
  and $\lambda_j,\lambda_j'$ chosen as in \Cref{thm:greedy-cd-privacy}
  to ensure $\epsilon,\delta$-DP. We denote by $j\in[p]$ the
  coordinate chosen at this step $t$, and by
  $j^* = \argmax_{j\in[p]}\abs{\nabla_j f(w^t)}/\sqrt{M_j}$ the coordinate
  that would have been chosen without noise. The following inequality
  holds
  \label{gcd-utility-proof:descent-lemma}
  \begin{align}
    f(w^{t+1}) - f(w^*)
    & \le f(w^t) - f(w^*)
      - \frac{1}{8} \norm{\nabla f(w^t)}_{M^{-1},\infty}^2
      + \frac{1}{2M_{j}} \abs{\eta_{j}^t}^2
      + \frac{1}{2M_{j}} \abs{\chi_{j}^t}^2
      + \frac{1}{4M_{j^*}} \abs{\chi_{j^*}^t}^2 \enspace.
  \end{align}
\end{lemma}
\begin{proof}
  The smoothness of $f$ gives a first inequality
  \begin{align}
    f(w^{t+1})
    & \le f(w^t)
      + \scalar{\nabla f(w^t)}{w^{t+1} - w^t}
      + \frac{1}{2} \norm{w^{t+1}-w^t}_M^2 \\
    & = f(w^t)
      - \frac{1}{M_j} \nabla_j f(w^t) (\nabla_j f(w^t) + \eta_j^t)
      + \frac{1}{2M_j} (\nabla_j f(w^t) + \eta_j^t)^2 \\
    & = f(w^t)
      - \frac{1}{M_j} \nabla_j f(w^t)^2
      - \frac{1}{M_j} \nabla_j f(w^t)  \eta_j^t
      + \frac{1}{2M_j} (\nabla_j f(w^t))^2 \nnlq
      + \frac{1}{M_j} \nabla_j f(w^t) \eta_j^t
      + \frac{1}{2M_j} (\eta_j^t)^2 \\
    & = f(w^t)
      - \frac{1}{2M_j} \nabla_j f(w^t)^2
      + \frac{1}{2M_j} (\eta_j^t)^2 \enspace.
      \label{gcd-utility-proof:descent-lemma:smoothness-upper-bound}
  \end{align}

  We will make the noisy gradient appear, so as to use the noisy
  greedy rule. To do so, we remark that the classical inequality
  $(a+b)^2 \le 2a^2 + 2b^2$ for any $a,b\in\RR$ implies that
  $-a^2 \le -\frac{1}{2}(a+b)^2 + b^2$. Applied with
  $a = \nabla_j f(w^t)/\sqrt{M_j}$ and
  $b = \chi_j^t/\sqrt{M_j}$, this results in
  \begin{align}
    - \frac{1}{2M_j} \nabla_j f(w^t)^2
    & \le - \frac{1}{4M_j} (\nabla_j f(w^t) + \chi_j^t)^2
      + \frac{1}{2M_j} (\chi_j^t)^2\enspace.
      \label{gcd-utility-proof:descent-lemma:upper-bound-gradient}
  \end{align}
  And, by the noisy greedy rule,
  $\frac{1}{\sqrt{M_{j^*}}} \abs{\nabla_{j^*} f(w^t) + \chi_{j^*}^t}
  \le \frac{1}{\sqrt{M_j}} \abs{\nabla_{j} f(w^t) + \chi_{j}^t}$. We
  replace
  in~\eqref{gcd-utility-proof:descent-lemma:upper-bound-gradient} and
  use the inequality $-a^2 \le -\frac{1}{2}(a+b)^2 + b^2$ with
  $a=(\nabla_{j^*} f(w^t) + \chi_{j^*})/\sqrt{M_{j^*}}$ and
  $b=-\chi_{j^*}/\sqrt{M_{j^*}}$ to obtain
  \begin{align}
    - \frac{1}{2M_j} \nabla_j f(w^t)^2
    & \le - \frac{1}{4M_{j^*}} (\nabla_{j^*} f(w^t) + \chi_{j^*}^t)^2
      + \frac{1}{2M_j} (\chi_j^t)^2 \\
    & \le - \frac{1}{8M_{j^*}} (\nabla_{j^*} f(w^t))^2
      + \frac{1}{4M_{j^*}} (\chi_{j^*}^t)^2
      + \frac{1}{2M_j} (\chi_j^t)^2\enspace.
  \end{align}
  The result follows
  from~\eqref{gcd-utility-proof:descent-lemma:smoothness-upper-bound}
  and
  $\frac{1}{M_{j^*}} (\nabla_{j^*} f(w^t))^2 = \norm{\nabla
    f(w^t)}_{M^{-1},\infty}^2$.
\end{proof}

\subsection{Utility for General Convex Functions}
\label{sec:util-gener-conv}

In this section, we derive an upper bound on the utility of DP-GCD for
convex objective functions. First, we use convexity of $f$ to upper
bound the decrease described in
\Cref{gcd-utility-proof:descent-lemma}. This gives \Cref
{gcd-utility-proof:descent-lemma-convex} in
\Cref{sec:descent-lemma-convex}, where the suboptimality gap
$f(w^{t+1})-f(w^*)$ at time $t+1$ is upper bound by a function of the
suboptimality gap $f(w^t)-f(w^*)$ at time $t$ and the noise injected
in step $t$. The novelty of our analysis lies in
\Cref{gcd-utility-proof:technical-lemma}, where examine the decrease
of the objective. Specifically, we show that either (i) $f(w^t)$ is
far from its minimum, and the suboptimality gap decreases with high
probability, either (ii) $f(w^t)$ is close to its minimum, then all
future iterates of DP-GCD will remain in a ball whose radius is
determined by the variance of the noise. This observation is essential
for proving the utility results stated in \Cref{sec:utility}.

\subsubsection{Descent Lemma for Convex Functions}
\label{sec:descent-lemma-convex}

\begin{lemma}
  \label{gcd-utility-proof:descent-lemma-convex}
  Under the hypotheses of \Cref{gcd-utility-proof:descent-lemma}, for
  a convex objective function $f$, we have
  \begin{align}
    f(w^{t+1}) - f(w^*)
    & \le f(w^t) - f(w^*)
      - \frac{(f(w^t) - f(w^*))^2}{8\norm{w^t - w^*}_{M,1}^2} \nnlq
      + \frac{1}{2M_{j}} \abs{\eta_{j}^t}^2
      + \frac{1}{2M_{j}} \abs{\chi_{j}^t}^2
      + \frac{1}{4M_{j^*}} \abs{\chi_{j^*}^t}^2 \enspace.
  \end{align}
\end{lemma}

\begin{proof}
  Since $f$ is convex, it holds that
  \begin{align}
    f(w^*)
    & \ge f(w^t) + \scalar{\nabla f(w^t)}{w^* - w^t}\enspace.
  \end{align}
  After reorganizing the terms, we can upper bound them using Hölder's
  inequality
  \begin{align}
    f(w^t) - f(w^*)
    & \le
      \scalar{\nabla f(w^t)}{w^t - w^*} \\
    & \le
      \norm{\nabla f(w^t)}_{M^{-1},\infty} \norm{w^t - w^*}_{M,1}
      \label{gcd-utility-proof:convex-norm-grad-proof-holder}
      \enspace,
  \end{align}
  where the second inequality holds since $\norm{\cdot}_{M,1}$ and
  $\norm{\cdot}_{M^{-1},\infty}$ are conjugate norms. We now divide
  \eqref{gcd-utility-proof:convex-norm-grad-proof-holder} by
  $\norm{w^t - w^*}_{M,1}$, square it and reorganize to get
  $- \norm{\nabla f(w^t)}_{M^{-1},\infty}^2 \le - \frac{(f(w^t) -
    f(w^*))^2}{\norm{w^t - w^*}_{M,1}^2}$. Replacing in
  \Cref{gcd-utility-proof:descent-lemma} gives the result.
\end{proof}

\subsubsection{Key Lemma on the Behavior of DP-GCD's Iterates}
\label{sec:key-lemma-behavior}

Now that we have an inequality in the form of
\Cref{gcd-utility-proof:descent-lemma-convex}, we prove that iterates
of DP-GCD converge to a vicinity of the optimum. In the general lemma below,
think of
$\xi_t$ as $f(w^t)-f(w^*)$ and of $\beta$ as the variance of the
term. This result will be combined with
\Cref{dp-gcd:techical-lemma-chernoff} to obtain high-probability
bounds.

\begin{lemma}
  \label{gcd-utility-proof:technical-lemma}
  Let $\{c_t\}_{t\ge 0}$ and $\{\xi_t\}_{t\ge 0}$ be two sequences of
  positive values that satisfy, for all $t \ge 0$,
  \begin{align}
    \label{gcd-utility-proof:technical-lemma:assumption-ineq}
    \xi_{t+1} \le \xi_t - \frac{\xi_t^2}{c_t} + \beta,
  \end{align}
  such that if $\xi_{t} \le \xi_0$ then $c_t \le c_0$.  Assume that
  $\beta \le c_0$ and $\xi_0 \ge 2 \sqrt{\beta c_0}$. Then:
  \begin{enumerate}
  \item For all $t>0$, $c_t \le c_0$, and there exists $t^*>0$ such
    that $\xi_{t+1}\le\xi_t$ if $t < t^*$ and
    $\xi_t\le 2\sqrt{\beta c_0}$ if $t \ge t^*$.
  \item For all $t \ge 1$,
    $\xi_t \le \frac{c_0}{t} + 2\sqrt{\beta c_0}$.
  \end{enumerate}
\end{lemma}

\begin{proof}
  1. Assume that for $t\ge 0$,
  $\sqrt{\beta c_0} \le \xi_t \le \xi_0$. Then,
  \begin{align}
    \xi_{t+1}
    \le \xi_t - \frac{\xi_t^2}{c_t} + \beta
    \le \xi_t - \frac{\sqrt{\beta c_0}^2}{c_0} + \beta
    = \xi_t\enspace,
    \label{dp-gcd:techical-lemma-first-ineq}
  \end{align}
  where the second inequality comes from
  $\xi_t \ge \sqrt{\beta c_0}$ and $\xi_t \le \xi_0$ (which implies
  $c_t \le c_0$). We now define the following value $t^*$, which
  defines the point of rupture between two regimes for $\xi_t$:
  \begin{align}
    t^* = \min \left\{ t \ge 0 \,\middle|\, \xi_t \le \sqrt{\beta c_0} \right\}\enspace.
  \end{align}
  Let $t < t^*$, assume that $\xi_t \le \xi_0$,
  then~\eqref{dp-gcd:techical-lemma-first-ineq} holds, that is
  $\xi_{t+1} \le \xi_t \le \xi_0$. By induction, it follows that for
  all~$t < t^*$, $\xi_{t+1} \le \xi_t \le \xi_0$ and $c_t \le c_0$.

  Remark now that $\xi_{t^*} \le \sqrt{\beta c_0}$, we prove by
  induction that $\xi_t$ stays under $2\sqrt{\beta c_0}$ for
  $t\ge t^*$.  Assume that for $t \ge t^*$,
  $\xi_t \le 2\sqrt{\beta c_0}$. Then, there are two
  possibilities. If $\xi_{t} \le \sqrt{\beta c_0}$, then
  \begin{align}
    \xi_{t+1}\le \xi_t - \frac{\xi_t^2}{c_t} + \beta \le \sqrt{\beta
    c_0} + \beta \le 2\sqrt{\beta c_0}\enspace,
  \end{align}
  and $\xi_{t+1} \le 2\sqrt{\beta c_0}$. Otherwise,
  $\sqrt{\beta c_0} \le \xi_t \le 2\sqrt{\beta c_0} \le \xi_0$
  and~\eqref{dp-gcd:techical-lemma-first-ineq} holds, which gives
  $\xi_{t+1} \le \xi_t \le 2\sqrt{\beta c_0}$. We proved that for
  $t \ge t^*$, $\xi_t \le 2\sqrt{\beta c_0}$, which
  concludes the proof of the first part of the lemma.

  2. We start by proving this statement for $0 < t < t^*-1$. Define
  $\omega = \frac{2u}{c_0}$ and $u = \sqrt{\beta c_0}$. The
  assumption on $\xi_t$ implies, by the first part of the lemma,
  $\xi_{t+1} \le \xi_t - \frac{\xi_t^2}{c_t} + \beta \le \xi_t -
  \frac{\xi_t^2}{c_0} + \beta$, which can be rewritten
  \begin{align}
    \xi_{t+1} - u \le (1 - \omega) (\xi_t - u) - \frac{(\xi_t - u)^2}{c_0} \enspace,
    \label{dp-gcd:techical-lemma-nice-recursion}
  \end{align}
  since
  $(1 - \omega) (\xi_t - u) - \frac{(\xi_t - u)^2}{c_0} = \xi_t -
  \omega \xi_t - u + \omega u - \frac{\xi_t^2}{c_0} - \frac{2\xi_t
    u}{c_0} - \frac{u^2}{c_0} = \xi_t - \frac{\xi_t^2}{c_0} - u +
  \omega u - \frac{u^2}{c_0}$, and
  $\omega u - \frac{u^2}{c_0} = \frac{u^2}{c_0} = \beta$. Since
  $t < t^*-1$, $\xi_{t+1} - u > 0$ and $\xi_t - u > 0$, we can thus
  divide ~\eqref{dp-gcd:techical-lemma-nice-recursion} by
  $(\xi_{t+1} - u)(\xi_{t} - u)$ to obtain
  \begin{align}
    \frac{1}{\xi_t - u}
    \le \frac{1 - \omega}{\xi_{t+1} - u} - \frac{\xi_t - u}{(\xi_{t+1} - u) c_0}
    \le \frac{1 - \omega}{\xi_{t+1} - u} - \frac{1}{c_0}
    \le \frac{1}{\xi_{t+1} - u} - \frac{1}{c_0}\enspace,
  \end{align}
  where the second inequality comes from
  $\xi_{t+1} - u \le \xi_t - u$ from the first part of the lemma. By
  applying this inequality recursively and taking the inverse of the
  result, we obtain the desired resuld
  $\xi_t \le \frac{c_0}{t} + \sqrt{\beta c_0} \le \frac{c_0}{t} +
  2\sqrt{\beta c_0}$ for all $0 < t < t^*$.

  For $t \ge t^*$, we have already proved that
  $\xi_t \le 2\sqrt{\beta c_0} \le \frac{c_0}{t} + 2\sqrt{\beta
    c_0}$, which concludes our proof.
\end{proof}

\subsubsection{Convex Utility Result}
\label{sec:conv-util-result}

\begin{restate-theorem}{\ref{thm:greedy-cd-utility}} (Convex Case)
  \label{gcd-utility-proof:general-convex}
  Let $\epsilon, \delta \in (0,1]$. Assume $\ell(\cdot; d)$ is a
  convex and $L$-component-Lipschitz loss function for all
  $d \in \cX$, and $f$ is $M$-component-smooth. Define $\cW^*$ the set
  of minimizers of $f$, and $f^*$ the minimum of $f$. Let
  $w_{priv}\in\mathbb{R}^p$ be the output of
  \Cref{algo:private-greedy-coordinate-descent} with step sizes
  $\gamma_j = {1}/{M_j}$, and noise scales $\lambda_1,\dots,\lambda_p,\lambda_1',\dots,\lambda_p'$
  set as in Theorem~\ref{thm:greedy-cd-privacy} (with $T$ chosen
  below) to ensure $(\epsilon,\delta)$-DP. Then, the following holds
  for $\zeta\in(0,1]$:
  \begin{align}
    f(w_{priv}) - f(w^*) \le \frac{8 R_M^2}{T} + \sqrt{32 R_M^2 \beta}
    \enspace,
  \end{align}
  where
  $\beta = \frac{2\lambda_{\max}^2}{M_{\min}}
  \log(\frac{8Tp}{\zeta})^2$, and
  $\displaystyle R_M = \max_{w\in\RR^p} \min_{w^*\in \cW^*} \left\{
    \norm{w-w^*}_{M,1} \mid f(w) \le f(w^0) \right\}$.  If we set
  $T = \Big(\frac{n^2 \epsilon^2 R_M^2 M_{\min}}{2^7L_{\max}^2
    \log(1/\delta)}\Big)^{{1}/{3}}$, then with probability at least
  $1-\zeta$,
  \begin{align}
    f(w^T) - f(w^0)
      = \widetilde O \Big( \frac{R_M^{4/3}L_{\max}^{2/3} \log(p/\zeta)}{M_{\min}^{1/3} n^{2/3}\epsilon^{2/3}} \Big)
      \enspace.
  \end{align}
\end{restate-theorem}

\begin{proof}
  Let $\xi_t = f(w^t) - f(w^*)$. We upper bound the following
  probability by the union bound, and the fact that for $t \ge 0$, the
  events $E_j^t: $~``coordinate $j$ is updated at step $t$'' for
  $j \in [p]$ partition the probability space:
  \begin{align}
    \prob{
    \exists t,
    \xi_{t+1} \ge
    \xi_t - \frac{\xi_t^2}{8\norm{w^t-w^*}_{M,1}^2} + \beta }
    \le
    \sum_{t=0}^{T-1} \prob{
    \xi_{t+1} \ge
    \xi_{t} - \frac{\xi_{t}^2}{8\norm{w^t-w^*}_{M,1}^2} + \beta } & \\
    =
    \sum_{t=0}^{T-1} \sum_{j=1}^p \prob{
    \xi_{t+1} \ge
    \xi_{t} - \frac{\xi_{t}^2}{8\norm{w^t-w^*}_{M,1}^2} + \beta~\land~E_j^t } \enspace. &
  \end{align}

  \Cref{gcd-utility-proof:descent-lemma-convex} gives
  $\xi_{t+1} \le \xi_t - \frac{\xi_t^2}{8 \norm{w^t-w^*}_{M,1}^2} +
  \frac{1}{2M_{j}} \abs{\eta_{j}^t}^2 + \frac{1}{2M_{j}}
  \abs{\chi_{j}^t}^2 + \frac{1}{4M_{j^*}} \abs{\chi_{j^*}^t}^2$. We
  thus have the following upper bound:
  \begin{align}
    \prob{
    \exists t,
    \xi_{t+1} \ge
    \xi_t - \tfrac{1}{8\norm{w^t-w^*}_{M,1}^2} \xi_t^2 + \beta }
    \le
    \sum_{t=0}^{T-1} \sum_{j=1}^p \prob{
    \tfrac{ \abs{\eta_{j}^t}^2}{2M_{j}}  + \tfrac{\abs{\chi_{j}^t}^2}{2M_{j}}
    + \tfrac{\abs{\chi_{j^*}^t}^2}{4M_{j^*}} \ge \beta } & \\
    \le
    \sum_{t=0}^{T-1} \sum_{j=1}^p \prob{
    \abs{\eta_{j}^t}^2 +
    \abs{\chi_{j}^t}^2 + \abs{\chi_{j^*}^t}^2
    \ge 2M_{\min} \beta } \enspace. &
  \end{align}
  By \Cref{dp-gcd:techical-lemma-chernoff} with
  $X_1=\eta_{j}^t \sim \Lap(\lambda_j)$,
  $X_2=\chi_{j}^t \sim \Lap(\lambda_j')$ and
  $X_3=\chi_{j^*}^t \sim \Lap(\lambda_{j^*}')$, it holds that
  \begin{align}
    \prob{ \abs{\eta_{j}^t}^2 +
    \abs{\chi_{j}^t}^2 + \abs{\chi_{j^*}^t}^2 \ge 2M_{\min} \beta }
    & \le 8 \exp\left( - \frac{\sqrt{2 M_{\min} \beta}}{2\lambda_{\max}} \right)
      = \frac{\zeta}{Tp} \enspace,
  \end{align}
  where the last equality comes from
  $\beta = \frac{2\lambda_{\max}^2}{M_{\min}}
  \log(\frac{8Tp}{\zeta})^2$. We have proved that
  \begin{align}
    \prob{
    \exists t,
    \xi_{t+1} \ge
    \xi_t - \frac{\xi_t^2}{8\norm{w^t-w^*}_{M,1}^2} + \beta }
    & \le
    \sum_{t=0}^{T-1} \sum_{j=1}^p \frac{\zeta}{Tp}
    = \zeta \enspace.
  \end{align}

  We now use our \Cref{gcd-utility-proof:technical-lemma}, with
  $\xi_t = f(w^t) - f(w^*)$; $c_0=8R_M^2$ and
  $c_t = 8\norm{w^t-w^*}_{M,1}^2$ for $t > 0$; and
  $\beta = \frac{2\lambda_{\max}^2}{M_{\min}}
  \log(\frac{8Tp}{\zeta})^2$.  These values satisfies the assumptions
  of \Cref{gcd-utility-proof:technical-lemma} since, by the definition
  of $R_M$, it holds that $c_t \le c_0$ whenever $\xi_t \le \xi_0$
  (\ie $f(w^t) - f(w^*) \le f(w^0) - f(w^*)$). Additionally,
  $f(w^0) - f (w^*) \ge \sqrt{32R_M^2\beta}$, therefore
  $f(w^0)-f(w^*) \ge 2\sqrt{\beta c_0}$, and $\beta \le c_0$.

  We obtain the result, with probability at least $1-\zeta$:
  \begin{align}
    f(w^t) - f(w^0)
    & \le \frac{c_0}{t} + 2\sqrt{\beta c_0}
      = \frac{8R_M^2}{t}
      + \frac{64 R_M L_{\max} \log(8Tp/\zeta) \sqrt{T \log(1/\delta)}}{\sqrt{M_{\min}}n\epsilon}
      \enspace.
  \end{align}
  For
  $T = \frac{R_M^{2/3} M_{\min}^{1/3} n^{2/3} \epsilon^{2/3}}{4 L_{\max}^{2/3} \log(1/\delta)^{1/3}}$, we obtain that, with
  probability at least $1-\zeta$,
  \begin{align}
    f(w^t) - f(w^0)
    & \le
      \frac{64 R_M^{4/3} L_{\max}^{2/3} \log(1/\delta)^{1/3}}{M_{\min}^{1/3} n^{2/3} \epsilon^{2/3}}  \log\Big( \frac{pR_M^{2/3} M_{\min}^{1/3} n^{2/3} \epsilon^{2/3}}{4 \zeta L_{\max}^{2/3} \log(1/\delta)^{1/3}}\Big)\enspace,
  \end{align}
  which is the result of the theorem.
\end{proof}

\subsection{Utility for Strongly-Convex Functions}
\label{sec:util-strongly-conv}

\subsubsection{A Key Inequality for Strongly-Convex Functions}
\label{sec:key-ineq-strongly}

We now prove a link between $f$'s largest gradient entry and the
suboptimality gap, under the assumption that there exists a unique
minimizer $w^*$ of $f$ that is $(\alpha,\tau)$-quasi-sparse. Note that this
assumption is not restrictive in general as any vector in $\mathbb{R}^p$ is $
(0,p)$-quasi-sparse, and for any $\tau$ there exists $\alpha>0$ such that the
vector is $(\alpha,\tau)$-quasi-sparse. We will
denote by
$\cW_{\tau, \alpha} \subseteq \RR^p$ the set of
$(\alpha,\tau)$-quasi-sparse vectors of $\RR^p$:
\begin{align}
  \cW_{\tau, \alpha} = \left\{
  w \in \RR^p \mid \abs{ \{ j \in [p] \mid \abs{w_j} \ge \alpha \} } \le \tau
  \right\}\enspace.
\end{align}
When $\alpha=0$, we simply write $\cW_\tau = \cW_{\tau,0}$, that is
the set of $\tau$-sparse vectors. We also define the associated
thresholding operator $\pi_\alpha$, that puts to $0$ the coordinates that are
smaller than $\alpha$, ``projecting'' vectors from
$\cW_{\tau,\alpha}$ to $\cW_\tau$, \ie for $w\in\RR^p$,
\begin{align}
  \pi_\alpha(w) =
  \begin{cases}
    0 & \text{if } \abs{w_j} \le \alpha \enspace, \\
    w_j & \text{otherwise} \enspace.
  \end{cases}
\end{align}
Importantly, restricting a function to $\tau$-sparse vectors changes
its strong-convexity parameter. Let $\tau \ge 0$ and $q \in \{1, 2\}$,
we say a function is $\mu_{M,q}^{(\tau)}$-strongly-convex when
restricted to $\tau$-sparse vectors if for all $\tau$-sparse vectors
$v,w \in \cW_\tau$,
\begin{align}
  f(w)
  & \ge f(v) + \scalar{\nabla f(v)}{w - v}
    + \frac{\mu_{M,q}^{(\tau)}}{2} \norm{w - v}_{M,q}^2
    \enspace.
\end{align}
Remark that when $\tau \ge p$, we recover the usual strong-convexity
parameters. The parameters \wrt $\ell_1$- and $\ell_2$-norms can be
compared using the following inequality \citep{fang2020Greed}, for all
$\tau \ge 0$,
\begin{align}
  \frac{1}{\tau} \mu_{M,2}^{(\tau)}
  & \le \mu_{M,1}^{(\tau)}
    \le \mu_{M,2}^{(\tau)}
    \enspace.
\end{align}
We are ready to prove \Cref{gcd-proof:lemma:upper-bd-grad}.

\begin{lemma}
  \label{gcd-proof:lemma:upper-bd-grad}
  Let $f : \RR^p \rightarrow \RR$ be a function that is
  $M$-component-smooth, and $\mu_{M,1}^{(\tau)}$-strongly-convex \wrt
  $\norm{\cdot}_{M,1}$ when restricted to $\tau$-sparse vectors, for
  $\tau \ge 0$. Assume that the unique minimizer $w^*$ of $f$ is
  $(\tau, \alpha)$-quasi-sparse, for $\alpha, \tau \ge 0$. Let
  $w^t \in \RR^p$ be a $t$-sparse vector for some $t \ge 0$. Then we
  have
  \begin{align}
    - \frac{1}{2} \norm{\nabla f(w^t)}_{M^{-1},\infty}
    & \le - \mu_{M,1}^{(t+\tau)} (f(w^t) - f(w^*))
      + \frac{1}{2}M_{\max}\mu_{M,1}^{(t+\tau)}(p - \tau) \alpha^2\enspace.
  \end{align}
\end{lemma}

\begin{proof}
  Let $w^t \in \RR^p$ be a $t$-sparse vector. Remark that $w^*$ is
  $(\alpha,\tau)$-quasi-sparse, meaning that $\pi_\alpha(w^*)$ is
  $\tau$-sparse. The union of $w^t$ and $\pi_\alpha(w^*)$'s supports
  ($\supp(w^t)$ and $\supp(\pi_\alpha(w^*))$) thus satisfies
  $\abs{\supp(w) \cup \supp(\pi_\alpha(w^*))} \le t + \tau$.  As the
  function $f$ is $\mu_{M,1}^{(t+\tau)}$-strongly-convex with respect
  to $\norm{\cdot}_{M,1}$ and $t + \tau$ sparse vector,
  \begin{align}
    f(\pi_\alpha(w))
    \ge f(w^t)
    + \scalar{\nabla f(w^t)}{\pi_\alpha(w) - w^t}
    + \frac{\mu_{M,1}^{(t+\tau)}}{2} \norm{\pi_\alpha(w) - w^t}_{M, 1}^2
    \enspace.
  \end{align}
  Since $\pi_\alpha : \cW_{\tau,\alpha} \rightarrow \cW_{\tau,0}$ is
  surjective, minimizing this equation for $w \in \cW_{\tau, \alpha}$
  on both sides gives
  \begin{align}
    \inf_{w \in \cW_{\tau}} f(w)
    & \ge f(w^t)
      - \sup_{w \in \cW_{\tau,\alpha}} \left\{
      \scalar{- \nabla f(w^t)}{w^t - \pi_\alpha(w)}
      - \frac{\mu_{M,1}^{(t+\tau)}}{2} \norm{\pi_\alpha(w) - w^t}_{M,1}^2
      \right\} \\
    & \ge f(w^t)
      - \sup_{w \in \RR^p} \left\{
      \scalar{- \nabla f(w^t)}{w^t - w}
      - \frac{\mu_{M,1}^{(t+\tau)}}{2} \norm{w - w^t}_{M,1}^2
      \right\}
      \enspace.
  \end{align}
  The second term corresponds to the conjugate of the
  function
  $\frac{1}{2} \norm{\cdot}_{M,1}^2$, that is
  $\frac{1}{2} \norm{\cdot}_{M^{-1},\infty}^2$
  \citep{boyd2004Convex}. This gives
  \begin{align}
    \inf_{w \in \cW_{\tau}} f(w)
    & \ge f(w^t)
      - \left(
      \frac{\mu_{M,1}^{(t+\tau)}}{2} \norm{\cdot}_1^2
      \right)^* (-\nabla f(w'))
    \\
    & = f(w^t) - \frac{1}{2\mu_{M,1}^{(t+\tau)}} \norm{\nabla f(w')}_{M^{-1},\infty}^2
      \enspace.
  \end{align}
  Finally, $w^*$ is the minimizer of $f$ (which is convex), thus
  $\nabla f(w^*)=0$. The smoothness of $f$ gives, for any $w\in\RR^p$,
  $f(w) \le f(w^*) + \frac{1}{2} \norm{w - w^*}_{M,2}^2$. Hence
  \begin{align}
    \inf_{w \in \cW_\tau} f(w)
    \le f(w^*) + \inf_{w \in \cW_\tau}
    \frac{1}{2} \norm{w - w^*}_{M,2}^2.
    \le f(w^*) + \frac{1}{2} \norm{\pi_\alpha(w^*) - w^*}_{M,2}^2 \enspace,
  \end{align}
  where the second inequality comes from
  $\pi_\alpha(w^*) \in \cW_\tau$, since $w^* \in \cW_{\tau,
    \alpha}$. It remains to observe that
  $\norm{\pi_\alpha(w^*) - w^*}_{M,2}^2 \le M_{\max}(p-\tau) \alpha^2$ to get
  the result.
\end{proof}

\begin{corollary}
  \label{gcd-proof:corr-lemma:upper-bd-grad}
  For $\tau$-sparse vectors, we have $\alpha=0$ and thus
  $(p-\tau) \alpha = 0$. \Cref{gcd-proof:lemma:upper-bd-grad} can thus
  be simplified as
  \begin{align}
    - \frac{1}{2} \norm{\nabla f(w^t)}_{M^{-1},\infty}^2
    & \le - \mu_{M,1}^{(t+\tau)} (f(w^t) - f(w^*))
      \enspace.
  \end{align}
  When vectors are not sparse ($\tau = p$), we recover the inequality
  $- \frac{1}{2} \norm{\nabla f(w^t)}_{M^{-1},\infty}^2 \le -
  \mu_{M,1} (f(w^t) - f(w^*))$.
\end{corollary}

\subsubsection{General Strongly-Convex Utility Result}
\label{sec:gener-util-result}

\begin{restate-theorem}{\ref{thm:greedy-cd-utility}} (Strongly-Convex
  Case) Let $\epsilon, \delta \in (0,1]$. Assume $\ell(\cdot; d)$ is a
  $\mu_{M,1}$-strongly-convex \wrt $\norm{\cdot}_{M,1}$ and
  $L$-component-Lipschitz loss function for all $d \in \cX$, and $f$
  is $M$-component-smooth. Let $\cW^*$ be the set of minimizers of
  $f$, and $f^*$ the minimum of $f$. Let $w_{priv}\in\mathbb{R}^p$ be
  the output of \Cref{algo:private-greedy-coordinate-descent} with
  step sizes $\gamma_j = {1}/{M_j}$, and noise scales
  $\lambda_1,\dots,\lambda_p,\lambda_1',\dots,\lambda_p'$ set as in
  Theorem~\ref{thm:greedy-cd-privacy} (with $T$ chosen below) to
  ensure $(\epsilon,\delta)$-DP. Then, the following holds for
  $\zeta\in(0,1]$:
  \begin{align}
    f(w^T) - f(w^*)
    \le (1 - \frac{\mu_{M,1}}{2})^T (f(w^0) - f(w^*))
    + \frac{64 T L_{\max}^2 \log(1/\delta)}{M_{\min} \mu_{M,1} n^2 \epsilon^2} \log(\frac{2Tp}{\zeta})
    \enspace.
  \end{align}
  If we set
  $T = \frac{2}{\mu_{M,1}} \log(\frac{M_{\min} \mu_{M,1} n^2 \epsilon^2
    (f(w^0)-f(w^*)}{32 L_{\max}^2 \log(1/\delta)})$, then with
  probability at least $1-\zeta$,
  \begin{align}
    f(w^T) - f(w^*)
    & = \widetilde O\Big( \frac{L_{\max}^2 \log(p/\zeta)}{M_{\min} \mu_{M,1}^2 n^2\epsilon^2} \Big)
    \enspace.
  \end{align}
\end{restate-theorem}

\begin{proof}
  When $f$ is $\mu_{M,1}$-strongly-convex \wrt the norm
  $\norm{\cdot}_{M,1}$, \Cref{gcd-proof:corr-lemma:upper-bd-grad} with
  $\tau=p$ and $\alpha=0$ (which holds for any vector) yields
  \begin{align}
    - \frac{1}{2} \norm{\nabla f(w^t)}_{M^{-1},\infty}^2
    & \le - \mu_{M,1} (f(w^t) - f(w^*))
      \enspace.
  \end{align}
  We replace this in \Cref{gcd-utility-proof:descent-lemma} to obtain
  \begin{align}
    f(w^{t+1}) - f(w^*)
    & \le (1 - \frac{\mu_{M,1}}{4}) (f(w^t) - f(w^*))
      + \frac{1}{2M_{j}} \abs{\eta_{j}^t}^2
      + \frac{1}{2M_{j}} \abs{\chi_{j}^t}^2
      + \frac{1}{4M_{j^*}} \abs{\chi_{j^*}^t}^2 \enspace.
  \end{align}
  Analogously to the proof of \Cref{gcd-utility-proof:general-convex},
  we define $\xi_t = f(w^t) - f(w^*)$ for all $0 \le t \le T$, and
  show that
  $\prob{ \exists t, \xi_{t+1} \ge (1 - \frac{\mu_{M,1}}{4}) \xi_t +
    \beta } \le \zeta/Tp$, with
  $\beta = \frac{2\lambda_{\max}^2}{M_{\min}}
  \log(\frac{8Tp}{\zeta})^2$. This yields that, with probability at
  least $1-\zeta$,
  \begin{align}
    f(w^T) - f(w^*)
    & \le (1 - \frac{\mu_{M,1}}{4})^T (f(w^0) - f(w^*))
      + \sum_{t=0}^{T-1}  (1 - \frac{\mu_{M,1}}{4})^{T-t} \beta \\
    & \le (1 - \frac{\mu_{M,1}}{4})^T (f(w^0) - f(w^*))
      + \frac{4}{\mu_{M,1}} \frac{32 T L_{\max}^2 \log(1/\delta)}{M_{\min} n^2 \epsilon^2}
      \log\Big(\frac{8Tp}{\zeta}\Big)^2 \enspace,
  \end{align}
  With
  $T = \frac{4}{\mu_{M,1}} \log\Big( \frac{\mu_{M,1} M_{\min} n^2 \epsilon^2
    (f(w^0)-f(w^*))}{128L_{\max}^2\log(1/\delta)\log(8p/\zeta)} \Big)$
  we have, with probability at least $1-\zeta$,
  \begin{align}
    f(w^T) - f(w^*)
    & \le
    \frac{128 L_{\max}^2\log(1/\delta)\log(8p/\zeta)^2}{\mu_{M,1}M_{\min}n^2\epsilon^2} \nnlq
      + \frac{512 L_{\max}^2\log(1/\delta)\log(8Tp/\zeta)^2}{\mu_{M,1}^2M_{\min}n^2\epsilon^2} \log\Big( \frac{\mu_{M,1} M_{\min} n^2 \epsilon^2
    (f(w^0)-f(w^*))}{128L_{\max}^2\log(1/\delta)\log(8p/\zeta)^2} \Big)
    \enspace,
  \end{align}
  which is the desired result.
\end{proof}

%% file: sup-fast-initial.tex
\subsubsection{Better Utility for Quasi-Sparse Solutions}
\label{sec:fast-init-conv-sup}

\begin{restate-theorem}{\ref{gcd:fast-initial-convergence}}
  \label{gcd-utility-proof:better-utility-strongly-convex}
  Consider $f$ satisfying the hypotheses of
  \Cref{thm:greedy-cd-utility}, with
  \Cref{algo:private-greedy-coordinate-descent} initialized at
  $w^0=0$. We denote its output $w^T$, and assume that its iterates
  remain $s$-sparse for some $s \le p$.  Assume that, for all
  $\tau' \ge 0$, $f$ is $\mu_{M,1}^{(\tau')}$-strongly-convex \wrt
  $\norm{\cdot}_{M,1}$ for $\tau'$-sparse
  vectors %
  and $\mu_{M,2}$-strongly-convex \wrt $\norm{\cdot}_{M,2}$, and that
  the (unique) solution of problem~\eqref{eq:dp-erm} is
  $(\alpha,\tau)$-quasi-sparse for some $\alpha,\tau\ge 0$.  Let
  $T \ge 0$, $\zeta\in[0,1]$, and
  $\beta = \frac{2\lambda_{\max}^2}{M_{\min}} \log(TP/\zeta)^2$. Then
  for all $t \le T$ we have that, with probability at least
  $1 - \zeta$:
  \begin{align}
    f(w^T) & - f(w^*)
             \le \Big(1 - \frac{\mu_{M,1}^{(\min(s,T)+\tau)}}{4}\Big)^T (f(w^0) - f(w^*))
             + \frac{4(\min(s,T)+\tau)\beta}{\mu_{M,2}} + \frac{\min(s, T)+\tau}{8} (p - \tau) \alpha^2
    \label{gcd-utility-proof:better-utility-strongly-convex:ineq-1} \\
    & \le \Big(1 - \frac{\mu_{M,2}}{4(\min(s,T)+\tau)}\Big)^T (f(w^0) - f(w^*))
      + \frac{4(\min(s,T)+\tau)\beta}{\mu_{M,2}} + \frac{\min(s,T)+\tau}{8} (p - \tau) \alpha^2
      \enspace.
      \label{gcd-utility-proof:better-utility-strongly-convex:ineq-2}
  \end{align}
  Setting
  $T = \frac{s+\tau}{\mu_{M,2}} \log((f(w^0)-f^*)
  M_{\min}\mu_{M,2}n^2\epsilon^2/L^2)$, and assuming
  $\alpha^2 = O\left( L_{\max}^2 (s+\tau) / M_{\min}\mu_{M,2}^2
    pn^2\epsilon^2 \right)$, we obtain that with probability at least
  $1-\zeta$,
  \begin{align}
    f(w^T) - f^* = \widetilde O\left( \frac{L_{\max}^2}{M_{\min}}
    \frac{(s+\tau)^2 \log(2p/\zeta)}{\mu_{M,2} n^2\epsilon^2 } \right)
    \enspace.
  \end{align}

\end{restate-theorem}

\begin{proof}
  First, we remark that at each iteration, we change only one
  coordinate. Therefore, after $t$ iterations, the iterate $w^t$ is at
  most $t$-sparse. Since all iterates are also $s$-sparse, it is
  $\min(s,t)$-sparse.  Additionally, we assumed that $w^*$ is
  $(\tau, \alpha)$-almost-sparse. Therefore,
  \Cref{gcd-proof:lemma:upper-bd-grad} yields
  \begin{align}
    - \frac{1}{2} \norm{\nabla f(w^t)}_{M^{-1},\infty}
    & \le - \mu_{M,1}^{(\min(s,t)+\tau)} (f(w^t) - f(w^*))
      + \frac{\mu_{M,1}^{(\min(s,t)+\tau)}}{2} (p - \tau) \alpha^2\enspace,
  \end{align}
  and \Cref{gcd-utility-proof:descent-lemma} becomes
  \begin{align}
    f(w^{t+1}) - f(w^*)
    & \le (1 - \frac{\mu_{M,1}^{(\min(s,t)+\tau)} }{4}) (f(w^t) - f(w^*))
      + \frac{\mu_{M,1}^{(\min(s,t)+\tau)}}{8} (p - \tau) \alpha^2 \nnlq
      + \frac{1}{2M_{j}} \abs{\eta_{j}^t}^2
      + \frac{1}{2M_{j}} \abs{\chi_{j}^t}^2
      + \frac{1}{4M_{j^*}} \abs{\chi_{j^*}^t}^2 \enspace.
  \end{align}
  Then by Chernoff's equality, we obtain (similarly to the proof of
  \Cref{thm:greedy-cd-utility} for the convex case) that with
  probability at least $1-\zeta$, for $T \ge 0$,
  \begin{align}
    f(w^T) - f(w^*)
    & \le \prod_{t=0}^T \Big(1 - \frac{\mu_{M,1}^{(\min(s,t)+\tau)}}{4}\Big) (f(w^0) - f(w^*)) \nnlq
      + \sum_{t=0}^{T-1} \prod_{k=T-t}^{T} \Big(1 - \frac{\mu_{M,1}^{(\min(s,k)+\tau)}}{4}\Big)
      \Big( \beta + \frac{\mu_{M,1}^{(\min(s,t)+\tau)}}{8} (p - \tau) \alpha^2 \Big)
    \enspace.
  \end{align}
  Since for $k \in [T]$,
  $\mu_{M,1}^{\min(s,k)+\tau} \ge \mu_{M,1}^{\min(s,T)+\tau}$, we can
  further upper bound
  $\mu_{M,1}^{(\min(s,t)+\tau)} \le \mu_{M,1}^{(\tau)}$, and
  $1 - \frac{\mu_{M,1}^{(\min(s,t)+\tau)}}{4} \le 1 -
  \frac{\mu_{M,1}^{(\min(s,T)+\tau)}}{4}$ and
  \begin{align}
  \sum_{t=0}^{T-1} \prod_{k=T-t}^{T} \Big(1 -
  \frac{\mu_{M,1}^{(\min(s,k)+\tau)}}{4}\Big) \le \sum_{t=0}^{T-1} \Big(1 -
  \frac{\mu_{M,1}^{(\min(s,T)+\tau)}}{4}\Big)^t \le
  \frac{4}{\mu_{M,1}^{(\min(s,T)+\tau)}}\enspace,
  \end{align}
  which allows to simplify the above expression to
  \begin{align}
    f(w^T) & - f(w^*)
             \le \Big(1 - \frac{\mu_{M,1}^{(\min(s,T)+\tau)}}{4}\Big)^T (f(w^0) - f(w^*))
      + \frac{4}{\mu_{M,1}^{(\min(s,T)+\tau)}} \Big( \beta + \frac{\mu_{M,1}^{(\tau)}}{8} (p - \tau) \alpha^2 \Big) \\
    & \le \Big(1 - \frac{\mu_{M,2}}{4(\min(s,T)+\tau)}\Big)^T (f(w^0) - f(w^*))
      + \frac{4(\min(s,T)+\tau)}{\mu_{M,2}} \Big( \beta + \frac{\mu_{M,2}}{8} (p - \tau) \alpha^2 \Big) \\
    & \le \Big(1 - \frac{\mu_{M,2}}{4(\min(s,T)+\tau)}\Big)^T (f(w^0) - f(w^*))
      + \frac{4(\min(s,T)+\tau)\beta}{\mu_{M,2}} + \frac{\min(s,T)+\tau}{8} (p - \tau) \alpha^2
      \enspace,    \label{gcd-proof:thm-fast-init:utility-unbalanced}
  \end{align}
  where the second inequality follows from
  $\mu_{M,1}^{(\min(s,T)+\tau)} \ge
  \frac{\mu_{M,2}^{(\min(s,T)+\tau)}}{\min(s,T)+\tau} \ge
  \frac{\mu_{M,2}}{\min(s,T)+\tau}$ and
  $\mu_{M,1}^{(\tau)} \le \mu_{M,2}$. We have proven
  inequalities~\eqref{gcd-utility-proof:better-utility-strongly-convex:ineq-1}
  and \eqref{gcd-utility-proof:better-utility-strongly-convex:ineq-2}
  of the theorem.

  When
  $\alpha^2 = O\left( L_{\max}^2 (s+\tau) / M_{\min}\mu_{M,2}^2
    pn^2\epsilon^2 \right)$, the two additive terms
  of~\eqref{gcd-proof:thm-fast-init:utility-unbalanced} are
  $O((s+\tau)\beta/\mu_{M,2})$. Since
  $\min(s, T) + \tau \le s + \tau$, we choose
  $T = \frac{s+\tau}{\mu_{M,2}} \log((f(w^0)-f^*)
  M_{\min}\mu_{M,2}n^2\epsilon^2/L^2)$ to balance all the terms and
  obtain the result.
\end{proof}

%% file: sup-proximal-gcd.tex
\section{Greedy Coordinate Descent for Composite Problems}
\label{sec:greedy-coord-desc}

Consider the problem of privately approximating
\begin{align}
  \label{eq:dp-composite-erm}
  w^* \in
  \argmin_{w \in \mathbb{R}^p}
  \frac{1}{n} \sum_{i=1}^n \ell(w; d_i) + \psi(w),
\end{align}
where $D = (d_1, \dots, d_n)$
is a dataset of $n$ samples drawn from a universe $\cX$,
$\ell: \RR^p \times \cX \rightarrow \RR$ is a loss function which is convex
and smooth in $w$, and
$\psi: \RR^p \rightarrow \RR$ is a convex regularizer which is separable (\ie
$\psi(w) = \sum_{j=1}^p \psi_j(w_j)$) and typically nonsmooth (\eg
$\ell_1$-norm).

\begin{algorithm}[h]
  \caption{DP-GCD (Proximal Version): Private Proximal Greedy CD}
  \label{algo:private-prox-greedy-coordinate-descent}
  \begin{algorithmic}[1]
    \State \textbf{Input:} initial $w^0 \in \RR^p$, iteration count $T > 0, \forall j \in [p],$ noise scales $\lambda_j, \lambda_j'$, step sizes $\gamma_j > 0$.
    \For{$t = 0$ to $T-1$}
      \State Select $j_t$ by the noisy \GSs, \GSr or \GSq rule.
      \State $\displaystyle w^{t+1} = w^t + (\prox{\gamma_j\psi_j} (w^t - \gamma_j (\nabla_j f(w^t) + \eta_{j_t}^t)) - w_j^t) e_j$, \hfill  $\eta_{j_t}^t \sim \Lap(\lambda_{j_t})$.
    \label{algo:private-prox-greedy-coordinate-descent:update}
    \EndFor
    \State \Return $w^T$.
  \end{algorithmic}
\end{algorithm}

We propose a proximal
greedy algorithm to solve \eqref{eq:dp-composite-erm}, see
\Cref{algo:private-prox-greedy-coordinate-descent}. The proximal operator
is the following \citep[we refer to][for a detailed
discussion on proximal operator and related algorithms]{parikh2014Proximal}:
\begin{align}
  \prox_{\gamma \psi}(v) = \argmin_{x\in\RR^p} \{ \frac{1}{2} \norm{v - x}_2^2 + \gamma \psi(x) \}\enspace.
\end{align}

The same privacy guarantees as for the smooth DP-GCD algorithm hold since,
privacy-wise, the proximal step is a post-processing step.
We also adapt the
greedy selection rule to incorporate the non-smooth term. We can use
one of the following three rules
\begin{align}
  j_t & = \argmax_{j\in[p]} \min_{\xi_j\in\partial\psi_j(w_j)} \frac{1}{\sqrt{M_j}}\abs{\nabla_j f(w^t) + \eta_j^t + \xi_j}\tag{\texttt{GS-s}} \enspace, \\
  j_t & = \argmax_{j\in[p]} \sqrt{M_j} \abs{\prox_{\frac{1}{M_j}\psi_j}(w_j^t - \frac{1}{M_j}(\nabla_j f(w^t) + \eta_j^t) - w_j^t} \tag{\texttt{GS-r}} \enspace, \\
  j_t & = \argmax_{j\in[p]} \min_{\alpha\in\RR} \nabla_j f(w^t) \alpha + \frac{M_j}{2} \alpha^2 + \psi_j(w_j^t + \alpha) - \psi_j(w_j^t) \tag{\texttt{GS-q}} \enspace.
\end{align}
These rules are commonly considered in the non-private GCD literature 
\citep[see \eg][]{tseng2009Coordinate,shi2017Primer,karimireddy2019Efficient},
except for the noise $\eta_j^t$ and the rescaling in the \GSs and \GSr
rules.

%% file: sup-expe-details.tex
\section{Experimental Details}
\label{sec:experimental-details}

In this section, we provide more information about the experiments, such as
details on implementation, datasets and the hyperparameter grid we
use for each algorithm. We then
give the full results on our L1-regularized, non-smooth, problems,
with the three greedy rules (as opposed to \Cref{sec:experiments-1} where
we only plotted results for the
\GSr rule). Finally, we provide runtime plots.

\paragraph{Code and setup}
\label{sec:code-setup}

The algorithms are implemented in C++ for efficiency, together with a
Python wrapper for simple use. It is provided as supplementary.
Experiments are run on a computer with a Intel
(R) Xeon(R) Silver 4114
CPU @ 2.20GHz and 64GB of RAM, and took about 10 hours in total to run
(this includes all
hyperparameter tuning).

\paragraph{Datasets}
\label{sec:datasets}

The datasets we use are described in \Cref{tab:datasets-desc}. In
\Cref{fig:solution-nice}, we plot the histograms of the absolute value
of each problem solution's parameters. The purple line indicates the
value of $\alpha$ that ensures that the parameters of the solution are
$(\alpha,5)$-quasi-sparse. Note the logarithmic scale on the $y$-axis.
On the \logg, \loggg, \madelon, \sparse, \california and \dorothea
datasets, the solutions are very imbalanced. In these problems, a very
limited number of parameters stand out, and DP-GCD is able to exploit
this property. This illustrates the results from
\Cref{sec:fast-init-conv}, since DP-GCD can exploit this structure
even in quasi-sparse problems, where $\alpha$ is non zero. Conversely,
the \mtp solution is more balanced: the structural properties of this
dataset are not strong enough for DP-GCD to outperform its competitors.

\begin{figure*}[t]
  \captionsetup[subfigure]{justification=centering}
  \centering
  \begin{subfigure}{0.235\linewidth}
    \centering
    \includegraphics[width=\linewidth]{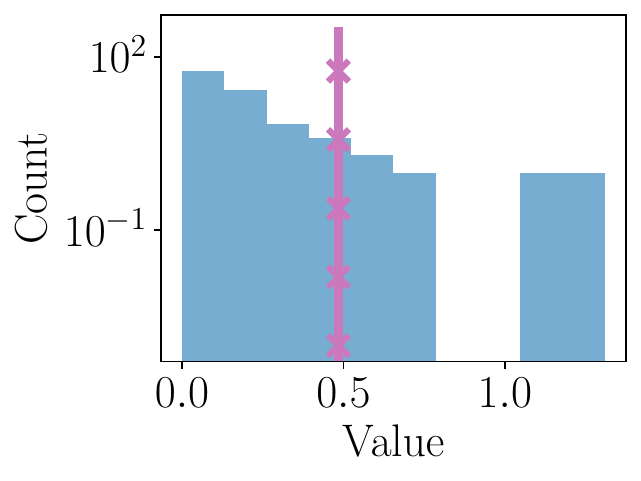}
    \caption{\texttt{log1} \\ Logistic + L2 \\  ($\lambda=1\text{e-}3$)}
    \label{fig:solution-skewed-1}
  \end{subfigure}%
  \begin{subfigure}{0.235\linewidth}
    \centering
    \includegraphics[width=\linewidth]{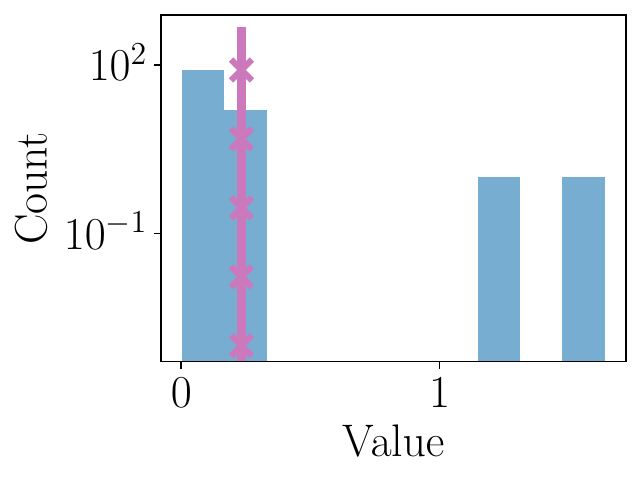}
    \caption{\texttt{log2} \\ Logistic + L2 \\  ($\lambda=1\text{e-}3$)}
    \label{fig:solution-skewed-2}
  \end{subfigure}%
  \begin{subfigure}{0.235\linewidth}
    \centering
    \includegraphics[width=\linewidth]{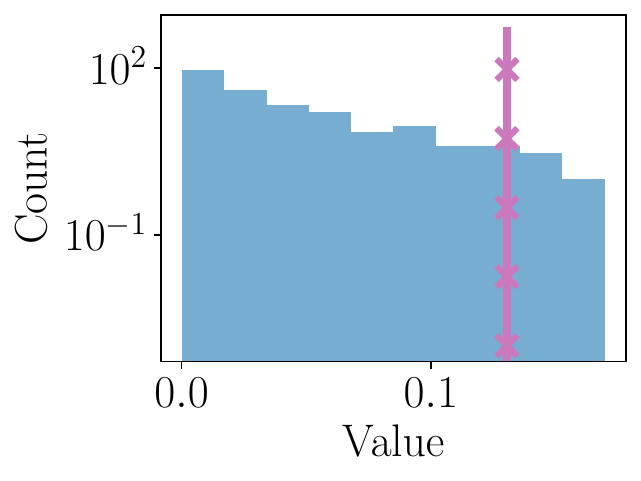}
    \caption{\texttt{mtp}\\ Least Squares + L2 \\  ($\lambda=5\text{e-}8$)}
    \label{fig:solution-mtp}
  \end{subfigure}%
  \begin{subfigure}{0.235\linewidth}
    \centering
    \includegraphics[width=\linewidth]{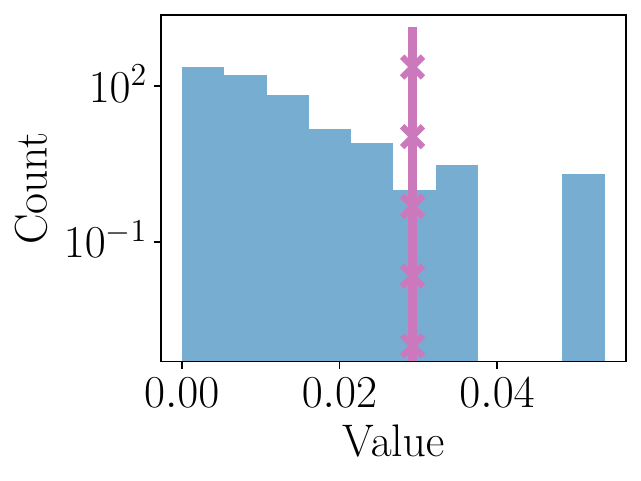}
    \caption{\texttt{madelon} \\ Logistic + L2 \\  ($\lambda=1$)}
    \label{fig:solution-madelon-l2}
  \end{subfigure}%

  \captionsetup[subfigure]{justification=centering}
  \centering
  \begin{subfigure}{0.235\linewidth}
    \centering
    \includegraphics[width=\linewidth]{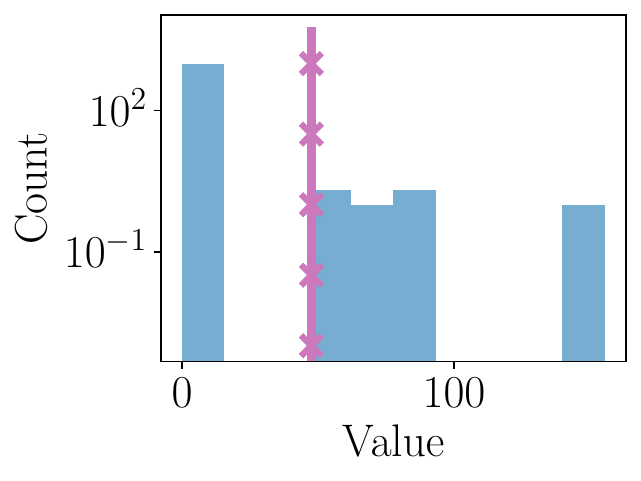}
    \caption{\texttt{square} \\ LASSO \\ ($\lambda=30$)}
    \label{fig:solution-sparse-lasso}
  \end{subfigure}%
  \begin{subfigure}{0.235\linewidth}
    \centering
    \includegraphics[width=\linewidth]{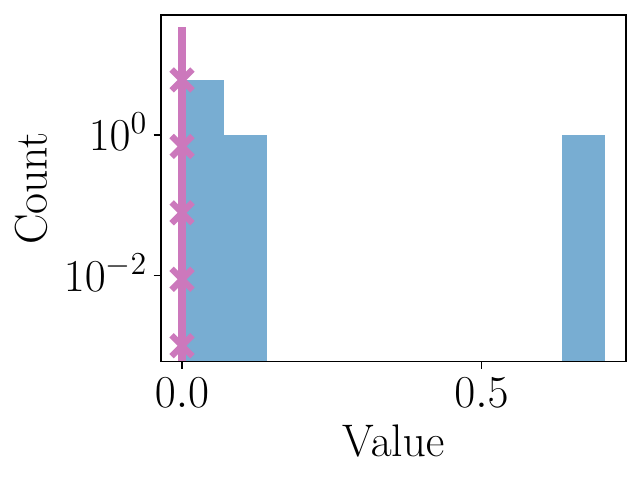}
    \caption{\texttt{california}\\ LASSO \\  ($\lambda=0.1$)}
    \label{fig:solution-california}
  \end{subfigure}%
  \begin{subfigure}{0.235\linewidth}
    \centering
    \includegraphics[width=\linewidth]{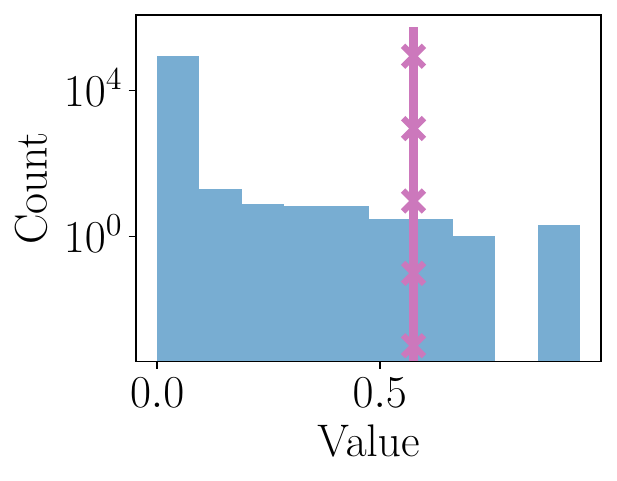}
    \caption{\texttt{dorothea}\\ Logistic + L1 \\  ($\lambda=0.01$)}
    \label{fig:solution-dexter}
  \end{subfigure}%
  \begin{subfigure}{0.235\linewidth}
    \centering
    \includegraphics[width=\linewidth]{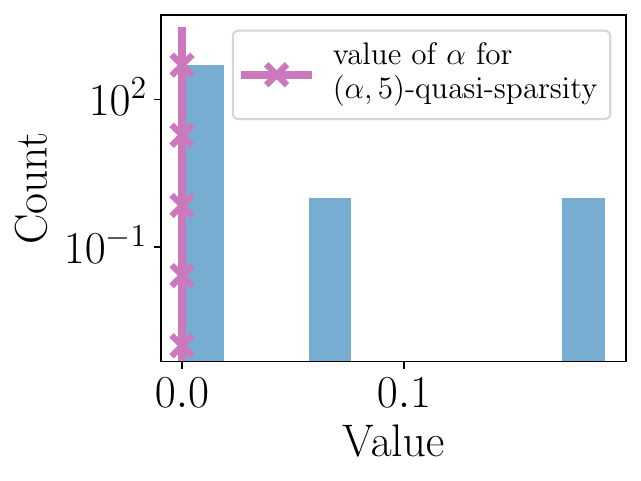}
    \caption{\texttt{madelon}\\ Logistic + L1 \\  ($\lambda=0.05$)}
    \label{fig:solution-madelon-l1}
  \end{subfigure}%

  \caption{Histograms of the absolute value of each problem solution's
    parameters. Purple line indicates the $\alpha$ for which the
    plotted vector is $(\alpha,5)$-quasi-sparse. Y-axis is
    logarithmic.}
  \label{fig:solution-nice}
\end{figure*}

\paragraph{Hyperparameters}
\label{sec:hyperparameters}

On all datasets, we use the same hyperparameter grid. For each
algorithm, we choose between roughly the same number of
hyperparameters. The number of passes on data represents $p$
iterations of DP-CD, $n$ iterations of DP-SGD, and $1$ iteration of
DP-GCD. The complete grid is described in
\Cref{tab:hyperparameter-grid}, and the chosen hyperparameters for
each problem and algorithm are given in
\Cref{table:chosen-hyperparameters}.

\begingroup

\renewcommand*{\arraystretch}{1.2}
\begin{table}
  \centering
  \caption{Hyperparameter grid used in our experiments.}
  \label{tab:hyperparameter-grid}
  \begin{tabular}{ccc}
    \toprule
    Algorithm & Parameter & Values \\
    \midrule
              & Passes on data & \texttt{[0.001, 0.01, 0.1, 1, 2, 3, 5, 10, 20]} \\
    DP-CD     & Step sizes & \texttt{np.logspace(-2, 1, 10)} \\
              & Clipping threshold & \texttt{np.logspace(-4, 6, 50)}  \\
    \midrule
              & Passes on data & \texttt{[0.001, 0.01, 0.1, 1, 2, 3, 5, 10, 20]} \\
    DP-SGD    & Step sizes & \texttt{np.logspace(-6, 0, 10)} \\
              & Clipping threshold & \texttt{np.logspace(-4, 6, 50)}  \\
    \midrule
              & Passes on data & \texttt{[1, 2, 4, 7, 10, 15, 20]} \\
    DP-GCD    & Step sizes & \texttt{np.logspace(-2, 1, 10)} \\
              & Clipping threshold & \texttt{np.logspace(-4, 6, 50)}  \\
    \bottomrule
  \end{tabular}
\end{table}
\endgroup

\paragraph{Recovery of the support}

In \Cref{tab:finding-of-support}, we report the number of coordinates
that are correctly/incorrectly identified as non-zero on $\ell_1$
regularized problems. Contrary to DP-SGD and DP-CD, DP-GCD never
incorrectly identifies a coordinate as non-zero. Additionally, the
suboptimality gap is lower for DP-GCD: its updates thus lead to better
solutions.

\begin{table}
  \centering
  \caption{ Coordinates correctly/incorrectly identified as non-zeros
    by each algorithm, and relative suboptimality gap
    $(f(w^{priv})-f^*)/f^*$ (averaged over 5~runs).}
  \label{tab:finding-of-support}
  \begin{tabularx}{\textwidth}{XXXXX}%
    \toprule
    & \texttt{square} & \texttt{california} & \texttt{dorothea} & \texttt{madelon}  \\
    \midrule
    $\|w^*\|_0$ & 7 & 3 & 72 & 3 \\
    DP-CD & 0 / 0 (0.75) & 3 / 2 (0.0024) & 1 / 1 (0.77) & 0 / 0 (0.0085) \\
    DP-SGD & 0 / 3 (0.75) & 3 / 5 (0.020) & 0 / 0 (0.78) & 0 / 0 (0.012) \\
    DP-GCD & 2 / 0 (0.35) & 2 / 0 (0.00056) & 1 / 0 (0.64) & 1 / 0 (0.0015) \\
    \bottomrule
  \end{tabularx}
\end{table}

\paragraph{Additional experiments on proximal DP-GCD}
\label{sec:addit-exper-prox}

In \Cref{fig:expe-prox}, we show the results of the proximal DP-GCD
algorithm, after tuning the hyperparameters with the grid described above for
each of the \GSs, \GSr and
\GSq rules.

The three rules seem to behave qualitatively the same on \sparse,
\dorothea and \madelon, our three high-dimensional non-smooth
problems. There, most coordinates are chosen about one time. Thus, as
described by \citet{karimireddy2019Efficient}, all the steps are
``\good'' steps (along their terminology): and on such good steps, the
three rules coincide. On the lower-dimensional dataset \california,
coordinates can be chosen more than one time, and ``\bad'' steps are
likely to happen. On these steps, the three rules differ.

\paragraph{Runtime}
Finally, we report the runtime of DP-GCD, in comparison with DP-CD and
DP-SGD in \Cref{fig:expe-time}, that is the counterpart of
\Cref{fig:expe-nice}, except with runtime on the $x$-axis. These
results confirm the fact that DP-GCD can be efficient, although its
iterations are expensive to compute. Indeed, in imbalanced problems,
the small number of iterations of DP-GCD enables it to run faster than
DP-SGD, and in roughly the same time as DP-CD, while improving
utility.

\begin{table}
  \centering
  \caption{
    Selected hyperparameters for every dataset and algorithm.
  }
  \label{table:chosen-hyperparameters}
  \begin{filecontents*}{plots/best_params_treated.csv}
Dataset,Loss,Algorithm,Passes on data,Clipping threshold,Step size
california,LeastSquares $+$ L1,DP-CD,5.0,2.02e$+$01,1.00e$+$00
square,LeastSquares $+$ L1,DP-CD,0.01,9.10e$+$03,1.00e$+$01
mtp,LeastSquares $+$ L2,DP-CD,3.0,2.02e$+$01,2.15e-02
madelon,Logistic $+$ L1,DP-CD,0.1,7.91e$+$00,2.15e$+$00
log1,Logistic $+$ L2,DP-CD,10.0,1.84e-01,1.00e$+$00
log2,Logistic $+$ L2,DP-CD,1.0,7.54e-01,2.15e$+$00
madelon,Logistic $+$ L2,DP-CD,10.0,1.21e$+$00,1.00e-01
dorothea,Logistic $+$ L1,DP-CD,3.0,4.50e-02,4.64e$+$00
california,LeastSquares $+$ L1,DP-SGD,20.0,1.26e$+$01,2.15e-05
square,LeastSquares $+$ L1,DP-SGD,0.01,4.94e$+$00,1.00e-04
mtp,LeastSquares $+$ L2,DP-SGD,20.0,1.26e$+$01,2.15e-05
madelon,Logistic $+$ L1,DP-SGD,10.0,6.87e-03,1.00e$+$00
log1,Logistic $+$ L2,DP-SGD,20.0,1.84e-01,4.64e-04
log2,Logistic $+$ L2,DP-SGD,20.0,1.84e-01,4.64e-04
madelon,Logistic $+$ L2,DP-SGD,20.0,1.84e-01,1.00e-04
dorothea,Logistic $+$ L1,DP-SGD,0.001,1.00e-04,1.00e-06
california,LeastSquares $+$ L1,DP-GCD,4,5.18e$+$01,1.00e$+$00
square,LeastSquares $+$ L1,DP-GCD,2,1.46e$+$04,2.15e$+$00
mtp,LeastSquares $+$ L2,DP-GCD,7,2.02e$+$01,4.64e-01
madelon,Logistic $+$ L1,DP-GCD,1,7.91e$+$00,2.15e$+$00
log1,Logistic $+$ L2,DP-GCD,10,3.09e$+$00,2.15e$+$00
log2,Logistic $+$ L2,DP-GCD,20,1.93e$+$00,4.64e-01
madelon,Logistic $+$ L2,DP-GCD,10,7.91e$+$00,1.00e$+$00
dorothea,Logistic $+$ L1,DP-GCD,2,1.26e$+$01,2.15e$+$00
\end{filecontents*}
\csvautobooktabular{plots/best_params_treated.csv}
\end{table}

\begin{figure*}[t]
  \captionsetup[subfigure]{justification=centering}
  \centering
  \begin{subfigure}{0.046\linewidth}
    \centering
    \vspace{-1.5em}
    \includegraphics[width=\linewidth]{plots/xlegend.pdf}
    \begin{minipage}{.1cm}
      \vfill
    \end{minipage}
  \end{subfigure}%
  \begin{subfigure}{0.235\linewidth}
    \centering
    \includegraphics[width=\linewidth]{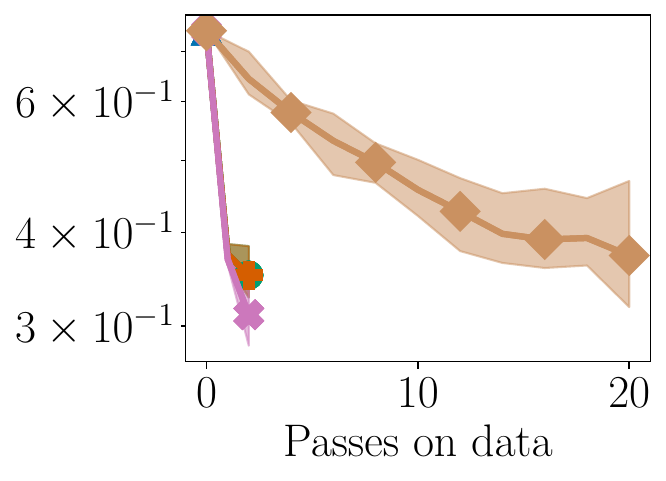}
    \caption{\texttt{sparse} \\ LASSO \\ ($\lambda=30$)}
    \label{fig:expe-sparse-lasso-gs}
  \end{subfigure}%
  \begin{subfigure}{0.235\linewidth}
    \centering
    \includegraphics[width=\linewidth]{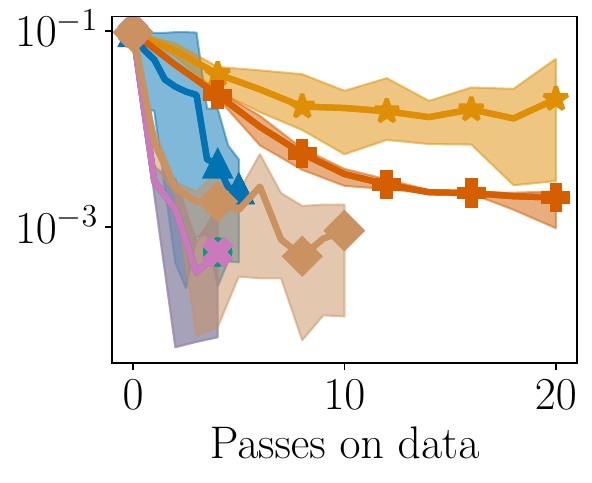}
    \caption{\texttt{california}\\ LASSO \\  ($\lambda=0.1$)}
    \label{fig:expe-california-gs}
  \end{subfigure}%
  \begin{subfigure}{0.235\linewidth}
    \centering
    \includegraphics[width=\linewidth]{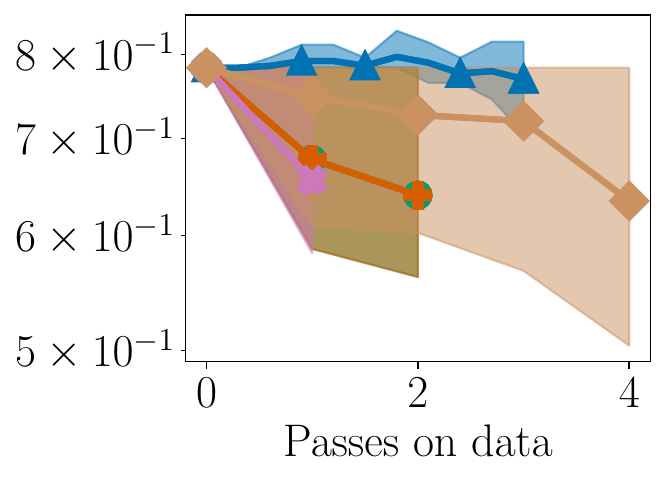}
    \caption{\texttt{dorothea}\\ Logistic + L1 \\  ($\lambda=0.01$)}
    \label{fig:expe-dorothea-gs}
  \end{subfigure}%
  \begin{subfigure}{0.235\linewidth}
    \centering
    \includegraphics[width=\linewidth]{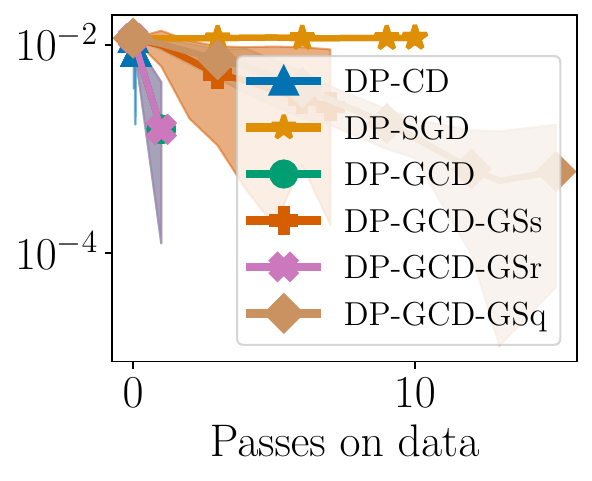}
    \caption{\texttt{madelon}\\ Logistic + L1 \\  ($\lambda=0.05$)}
    \label{fig:expe-madelon-l1-gs}
  \end{subfigure}%

  \caption{Relative error to non-private optimal for DP-CD, proximal
    DP-GCD (with \GSr, \GSs and \GSq rules) and DP-SGD on different
    problems. On the x-axis, $1$ tick represents a full access to the
    data: $p$ iterations of DP-CD, $n$ iterations of DP-SGD and $1$
    iteration of DP-GCD. Number of iterations, clipping thresholds and
    step sizes are tuned simultaneously for each algorithm. We report
    min/mean/max values over 5~runs.}
  \label{fig:expe-prox}
\end{figure*}

%% file: sup-expe-time.tex
\begin{figure*}[t]
  \captionsetup[subfigure]{justification=centering}
  \centering
  \begin{subfigure}{0.04\linewidth}
    \centering
    \vspace{-1.5em}
    \includegraphics[width=\linewidth]{plots/xlegend.pdf}
    \begin{minipage}{.1cm}
      \vfill
    \end{minipage}
  \end{subfigure}%
  \begin{subfigure}{0.235\linewidth}
    \centering
    \includegraphics[width=\linewidth]{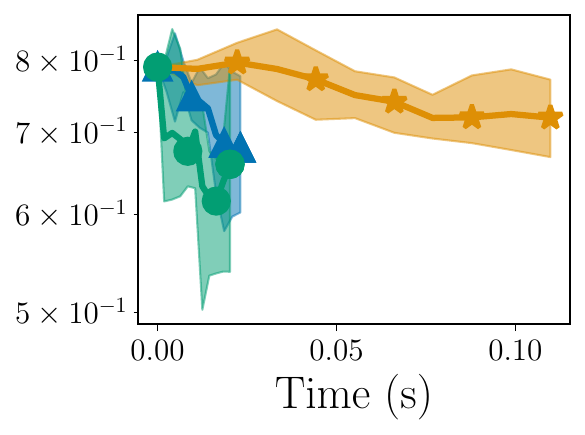}
    \caption{\texttt{log1} \\ Logistic + L2 \\  ($\lambda=1\text{e-}3$)}
    \label{fig:expe-time-skewed-1}
  \end{subfigure}%
  \begin{subfigure}{0.235\linewidth}
    \centering
    \includegraphics[width=\linewidth]{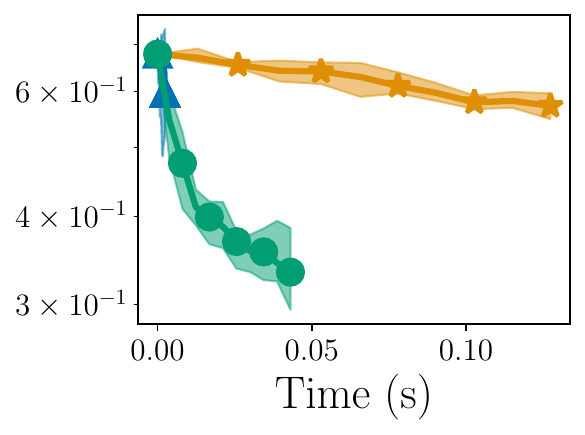}
    \caption{\texttt{log2} \\ Logistic + L2 \\  ($\lambda=1\text{e-}3$)}
    \label{fig:expe-time-skewed-2}
  \end{subfigure}%
  \begin{subfigure}{0.235\linewidth}
    \centering
    \includegraphics[width=\linewidth]{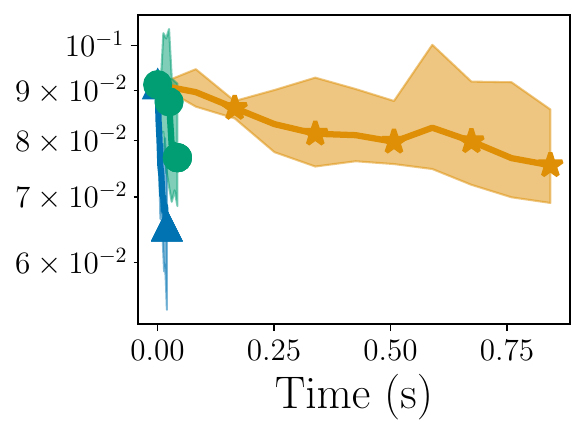}
    \caption{\texttt{mtp}\\ Least Squares + L2 \\  ($\lambda=5\text{e-}8$)}
    \label{fig:expe-time-mtp}
  \end{subfigure}%
  \begin{subfigure}{0.235\linewidth}
    \centering
    \includegraphics[width=\linewidth]{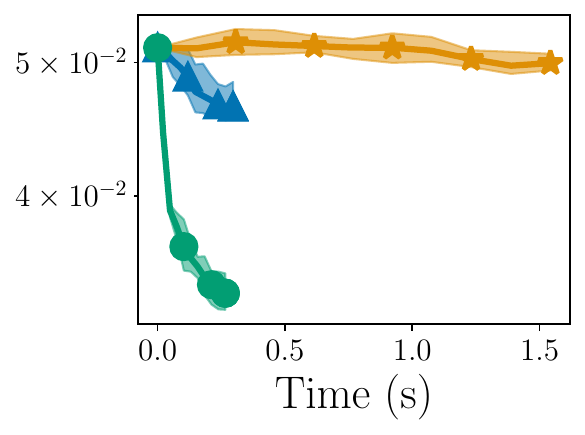}
    \caption{\texttt{madelon} \\ Logistic + L2 \\  ($\lambda=1$)}
    \label{fig:expe-time-madelon-l2}
  \end{subfigure}%

  \captionsetup[subfigure]{justification=centering}
  \centering
  \begin{subfigure}{0.04\linewidth}
    \centering
    \vspace{-1.5em}
    \includegraphics[width=\linewidth]{plots/xlegend.pdf}
    \begin{minipage}{.1cm}
      \vfill
    \end{minipage}
  \end{subfigure}%
  \begin{subfigure}{0.235\linewidth}
    \centering
    \includegraphics[width=\linewidth]{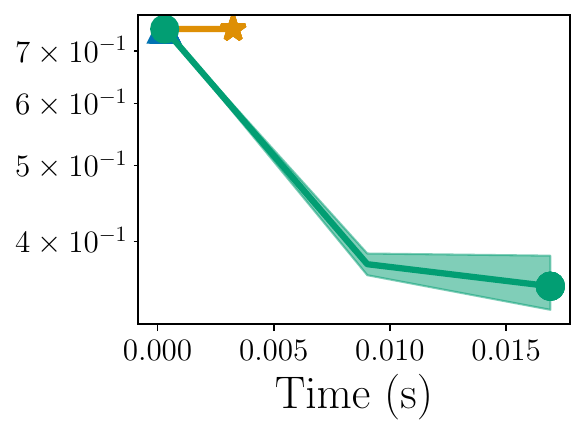}
    \caption{\texttt{sparse} \\ LASSO \\ ($\lambda=30$)}
    \label{fig:expe-time-sparse-lasso}
  \end{subfigure}%
  \begin{subfigure}{0.235\linewidth}
    \centering
    \includegraphics[width=\linewidth]{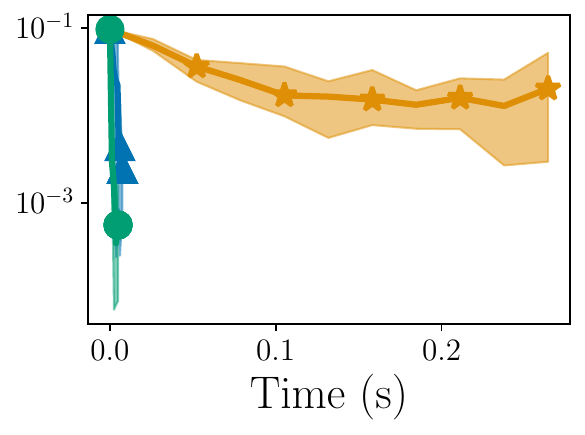}
    \caption{\texttt{california}\\ LASSO \\  ($\lambda=0.1$)}
    \label{fig:expe-time-california}
  \end{subfigure}%
  \begin{subfigure}{0.235\linewidth}
    \centering
    \includegraphics[width=\linewidth]{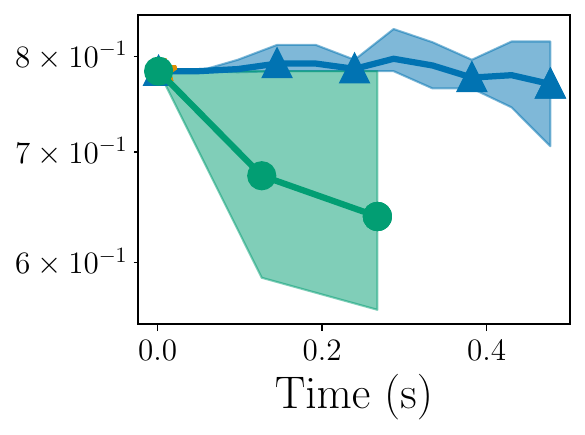}
    \caption{\texttt{dorothea}\\ Logistic + L1 \\  ($\lambda=0.01$)}
    \label{fig:expe-time-dorothea}
  \end{subfigure}%
  \begin{subfigure}{0.235\linewidth}
    \centering
    \includegraphics[width=\linewidth]{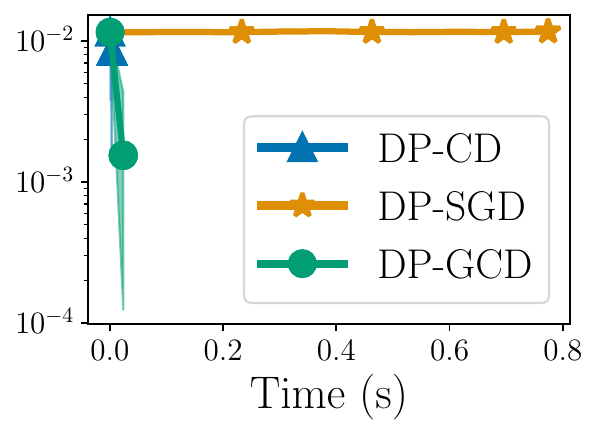}
    \caption{\texttt{madelon}\\ Logistic + L1 \\  ($\lambda=0.05$)}
    \label{fig:expe-time-madelon-l1}
  \end{subfigure}%

  \caption{Relative error to non-private optimal for DP-CD, DP-GCD and
    DP-SGD on different problems, as a function of running
    time. Number of iterations, clipping thresholds and step sizes are
    tuned simultaneously for each algorithm. We report min/mean/max
    values over 5~runs.}
  \label{fig:expe-time}
\end{figure*}